%% file: root.tex
\documentclass[journal]{IEEEtran}  

\IEEEoverridecommandlockouts                              

\usepackage{amsmath,amsfonts,amssymb,amsthm}
\usepackage{algorithm}
\usepackage{algpseudocode}
\usepackage{array}
\usepackage[caption=false,font=normalsize,labelfont=sf,textfont=sf]{subfig}
\usepackage{textcomp}
\usepackage{url}
\usepackage{verbatim}
\usepackage{graphicx}
\def\BibTeX{{\rm B\kern-.05em{\sc i\kern-.025em b}\kern-.08em
    T\kern-.1667em\lower.7ex\hbox{E}\kern-.125emX}}
\usepackage{balance}
\usepackage{hyperref}
\usepackage{xcolor}
\usepackage{multirow}
\usepackage{mathtools}
\usepackage{color,soul}
\usepackage{amsthm}
\usepackage{tikz}
\usetikzlibrary{arrows.meta, positioning, shapes.geometric, shapes.multipart, calc, fit, backgrounds}

\tikzset{arrow/.style={-Latex, line width=0.5pt}}

\tikzset{
  block/.style={
    draw, thick, rounded corners=1.5mm,
    align=center,
    minimum width=16mm,      
    minimum height=7mm,      
    inner xsep=2mm,          
    inner ysep=1.5mm,        
    font=\sffamily\footnotesize  
  },
  line/.style={-Latex, thick}
}

\tikzset{
  dashedbox/.style={
    draw, dashed, rounded corners=3mm, thick,
    inner sep=6pt
  }
}

\theoremstyle{plain}

\newtheorem{proposition}{Proposition}

\theoremstyle{definition}
\newtheorem{definition}{Definition}
\theoremstyle{remark}

\title{\LARGE \bf
Plan--and--Avoid: Real-Time Aircraft Trajectory Coordination in a Multi-Agent Environment}

\author{H. Emre Tekaslan$^{1}$ and Ella M. Atkins $^{2}$ and Natasha A. Neogi $^{3}$
\thanks{$^{1}$H. Emre Tekaslan is with the Kevin T. Crofton Department of Aerospace and Ocean Engineering,
        Virginia Polytechnic Institute and State University, Blacksburg, VA 24060, USA,
        {Corresponding author: \tt\small tekaslan@vt.edu}}%
\thanks{$^{2}$Ella M. Atkins is with the Kevin T. Crofton Department of Aerospace and Ocean Engineering,
        Virginia Polytechnic Institute and State University,
        Blacksburg, VA 24060, USA,
        {\tt\small ematkins@vt.edu}}%
\thanks{$^{3}$Natasha Neogi is with the NASA Langley Research Center, Hampton, VA 23681, USA,
        {\tt\small natasha.a.neogi@nasa.gov}}%
}

\begin{document}

\maketitle
\thispagestyle{empty}
\pagestyle{empty}

\input{Sections/0_Abstract}

\input{Sections/1_Intro}
\input{Sections/2_Prelim}
\input{Sections/3_Planning}
\input{Sections/4_Deconfliction}
\input{Sections/5_Algorithm}
\input{Sections/6_Application}
\input{Sections/7_Conclusion}

\bibliographystyle{IEEEtran}
\bibliography{references}


\end{document}

%% file: Sections/0_Abstract.tex
\begin{abstract}
This paper presents a real-time \textit{Plan--and--Avoid (PAA)} framework for coordinating cooperative multi-agent airspace operations around a declared priority trajectory. The priority trajectory represents an aircraft flight plan that must be preserved because of constrained maneuverability, an emergency, a mission-critical task, or assigned operational priority. The framework predicts uncertainty-aware, well-clear separation violations with surrounding traffic and, when the priority plan alone cannot maintain separation, generates vehicle-constrained unilateral advisories that modify nearby aircraft trajectories to maintain well-clear separation for all traffic. The approach is applicable to any declared priority trajectory. This paper demonstrates the \textit{Plan} component using a contingency landing planner to generate candidate priority trajectories.  PAA then identifies nearby aircraft passing too close to this priority trajectory and issues \textit{Avoid} resolution advisories to these aircraft. The framework is tested using real-world Automatic Dependent Surveillance-Broadcast (ADS-B) traffic from the Washington, D.C., airspace across more than 900 forced-landing cases, totaling over 140 hours of simulated flight. The PAA framework generates feasible cooperative advisories for all 575 unique conflict encounters, with a worst-case end-to-end response time of 5.7 s on a personal computer, including priority trajectory planning, advisory generation, and 1 s two-way datalink delay. In total, 93.5\% of generated advisories satisfy the 35 s RTCA DO--365 Detect-and-Avoid temporal threshold. These results demonstrate low-latency coordination for preserving priority trajectories while maintaining well-clear separation through real-time automated advisory generation. Future work will quantify advisory-induced delays and their operational impacts.
\end{abstract}


%% file: Sections/1_Intro.tex
\section{Introduction}
\label{sec:intro}
Cooperative multi-agent autonomy has gained increasing attention because coordinated operation can improve mission efficiency, coverage, and resilience through task allocation and distributed decision making \cite{10355051,11348992}. However, as the number of agents and mission duration increase, off-nominal events and operational priority constraints become more likely to affect the broader system. An emergency, time-critical objective, or a priority constraint affecting one agent can therefore become a system-level problem rather than an isolated, single-agent event. In such cases, multi-agent coordination must account not only for the priority agent's emergency recovery, mission, or operational objective, but also for the response of surrounding agents whose trajectories may be affected by the declared priority trajectory.

This issue is especially important as small Uncrewed Aircraft Systems (UAS), Advanced Air Mobility (AAM) vehicles, and conventional aircraft increasingly operate in shared or closely coupled airspace \cite{Raza2025}. In such environments, aircraft may differ substantially in size, performance, mission objective, autonomy level, and equipage, yet must still satisfy vehicle performance limits, operational constraints, and well-clear separation requirements. Once an aircraft declares a priority trajectory with limited flexibility, surrounding cooperative traffic may need to be coordinated to preserve separation, limit conflict propagation, and minimize disruption to the broader operation. Such trajectories may arise due to emergency situations, time-critical missions, constrained approach or departure procedures, or other operational constraints.  The aforementioned trajectories may be generated by an upstream planner, operator, or traffic management system.

Motivated by this need, this work develops a real-time \textit{Plan--and--Avoid (PAA)} framework for priority trajectory coordination in cooperative multi-agent airspace. PAA first evaluates a declared priority trajectory against known multi-agent intent and predicts uncertainty-aware loss-of-well-clear (LoWC) conflicts. When separation cannot be maintained by the priority trajectory alone, PAA generates vehicle-constrained resolution advisories for affected surrounding agents while preserving the declared priority trajectory. In this way, PAA treats priority operations as system-level coordination problems rather than single-aircraft planning problems. Emergency landing is used in this paper as a rigorous case study because the priority aircraft has limited maneuver and temporal flexibility and must preserve a feasible route to a safe landing site. The resulting PAA problem lies at the intersection of priority trajectory planning and cooperative conflict resolution. The priority route should minimize avoidable conflicts when flexibility exists, and surrounding traffic must be coordinated through feasible advisories when conflicts are inevitable.

\subsection{Background and Related Work}
\label{sec:background}
Most existing work on multi-agent autonomy addresses nominal coordination problems, including task allocation \cite{liu2022multi,yan2024cooperative}, path planning \cite{hu2023multi,drones8100537}, formation control \cite{pan2021improved}, and collision avoidance \cite{rezaee2024comprehensive}. These methods are valuable for improving efficiency and safety during planned operations, but emergency or priority-driven operations remain less developed. A recent review on UAS collision avoidance \cite{10477874} similarly notes that multi-UAS strategies often lack the responsiveness required for emergency maneuvers and real-time collision avoidance, particularly in environments with moving obstacles and other drones. Ref. \cite{drones8100537} considers fault-tolerant coverage path planning for multi-UAS operations with agent failures, but does not address the downstream conflicts caused by the failure or how surrounding agents should respond. In \cite{Conde2011}, UAS conflict resolution is performed through genetic algorithm based path modification; however, the results are limited to up to seven agents and do not consider declared priority trajectories with limited maneuver flexibility. These gaps motivate a coordination framework that treats priority operations as system-level events rather than isolated single-agent replanning problems.

The literature also provides important tools for conflict-aware path planning, including dynamic artificial potential fields for real-time air traffic collision avoidance \cite{8903321}, collision avoidance under obstacle uncertainty \cite{Lai2011}, onboard and learning-enabled avoidance methods \cite{9764175,karampinis2024ensuringuavsafetyvisiononly,11127284}, sampling-based planning with dynamic obstacles \cite{Lee2022,7911207,Aoude2013}, and Markov Decision Process (MDP)-based conflict resolution or route guidance \cite{ZHANG2021103123,9681334,weng2025realtimetrafficsimulationmanagement}. Related work has also considered four-dimensional first-come, first-served path planning using Conflict-Free A$^{\star}$. However, most of these methods are formulated for nominal mission execution, where the controlled vehicle remains fully functional and can repeatedly replan toward an independently optimized objective. Declared priority trajectories change this structure. Once a priority route is declared, the route may have limited flexibility, and conflict avoidance cannot rely only on modifying the priority aircraft trajectory. The problem therefore requires not only conflict-aware priority trajectory computation, but also a downstream mechanism for coordinating surrounding traffic when conflict-free planning is not possible.

Collision avoidance and conflict resolution have been extensively studied through methods such as reciprocal velocity obstacles (RVO) \cite{van2008reciprocal,chen2024reciprocal}, optimal reciprocal collision avoidance (ORCA) \cite{van2011reciprocal}, Model-Predictive Control (MPC) \cite{8950150}, Control Barrier Function (CBF)-based safety filters \cite{10160857}, Traffic Alert and Collision Avoidance System II (TCAS-II) \cite{Munoz2013}, Airborne Collision Avoidance System X (ACAS-X) \cite{holland2013optimizing}, and learning-based policies \cite{kahn2017uncertainty}. These approaches are effective for local separation assurance around nominal task trajectories, but they are not generally designed to preserve a declared priority trajectory while coordinating heterogeneous surrounding agents. Many formulations assume reciprocal maneuvering responsibility, direct control authority over all vehicles, or simplified dynamics, which limits their applicability when agents differ in vehicle type, operational role, and maneuverability. ORCA, for example, relies on a holonomic agent model \cite{10184914}, and its non-holonomic extensions remain reciprocal, similar in spirit to Detect-and-Avoid (DAA) methods such as TCAS-II and ACAS-X. These DAA systems are developed to satisfy minimum operational standards for DAA systems such as RTCA DO--365 \cite{RTCA_DO365_2017}, which defines high-level well-clear and alerting requirements. However, such requirements do not provide a system-level coordination framework for preserving a declared priority trajectory by issuing vehicle-constrained advisories to surrounding cooperative traffic. Additionally, CBF-based methods provide reactive local safety filtering \cite{hu2022decentralized}, whereas the present problem requires higher-level advisory selection over multiple possible maneuver classes. MPC can handle prediction and constraints, but centralized MPC over all affected agents leads to a large nonlinear optimization problem \cite{richards2004decentralized,bemporad2010decentralized}.  Learning-based approaches would require broad training coverage across encounter geometries, vehicle classes, and operational constraints and would be difficult to certify. The problem studied here therefore lies between trajectory planning and conflict resolution: the priority trajectory should be preserved when possible, while cooperative surrounding aircraft receive situation and vehicle-specific advisories to restore separation and limit conflict propagation.

\subsection{Contributions and Paper Organization}
\label{sec:contributions}
This paper introduces the PAA framework for real-time well-clear separation assurance around declared priority trajectories in cooperative multi-agent airspace. The framework integrates two stages. The \textit{Plan} stage evaluates candidate priority trajectories against surrounding multi-agent intent and uncertainty-aware well-clear constraints. The \textit{Avoid} stage generates unilateral conflict resolution advisories for surrounding cooperative agents when separation cannot be maintained by the priority trajectory alone. In the emergency landing demonstration, Gradient-Guided Search \cite{tekaslan_airspace}, a contingency landing planner for wing-lift aircraft, is extended with known multi-agent intent to produce conflict-aware priority trajectories. Remaining conflicts are then passed to the advisory layer, evaluating candidate trajectory transformations using kinematic trajectory models constrained by agent-specific operational limits. Each candidate advisory is also checked for secondary conflicts introduced by the advised maneuver. The proposed framework is demonstrated using real-world airspace data to show low-latency coordination of priority trajectories in realistic multi-agent environments.

The main contributions of this paper are threefold. First, a real-time Plan--and--Avoid architecture is developed for coordinating cooperative multi-agent airspace operations around a declared priority trajectory. This architecture is the core contribution of the paper and is not a simple composition of existing planning and avoidance methods. It couples a new priority trajectory evaluation stage with a new real-time advisory-generation stage into an end-to-end coordination framework. Second, the Plan stage extends existing contingency landing planning to consider dynamic multi-agent conflicts for the first time, enabling forced landing trajectories to be evaluated against surrounding traffic intent and uncertainty-aware well-clear constraints. Third, the Avoid stage introduces a hierarchical unilateral resolution advisory module that recovers well-clear separation by modifying surrounding agent trajectories while preserving the declared priority trajectory, accounting for vehicle-specific limits and secondary conflicts introduced by the advised maneuver. The developed software is made openly available\footnote{\url{https://github.com/tekaslan/Plan-and-Avoid}}.

The paper is structured as follows. Section \ref{sec:prelim} provides background on contingency landing planning and well-clear separation. Section \ref{sec:planning} extends the well-clear prediction model to account for uncertain position estimates and develops conflict-aware discrete search-based planning in a multi-agent environment. Sections \ref{sec:deconfliction} and \ref{sec:alg} describe the rule-based conflict resolution advisories for cooperative multi-agent operations and their implementation. Section \ref{sec:application} reports validation cases, a real-world airspace case, extensive benchmarking for conflict-aware path planning, and conflict resolution for all encounters in the benchmark solution set under real-world ADS-B-based multi-agent airspace. Finally, Sections \ref{sec:disc} and \ref{sec:conc} conclude the study.

%% file: Sections/2_Prelim.tex
\section{Preliminaries}
\label{sec:prelim}
This section reviews key concepts for real-time contingency landing planning followed by definition of the well-clear air traffic separation standard used for collision avoidance.

\subsection{Contingency Landing Planning with Discrete Search}
\label{sec:search_prelim}
A discrete search operator, in general terms, expands states with admissible actions to incrementally explore the state-space and find cost minimizing course of actions to reach a goal state starting from an initial state. When considered through the lens of path planning, a state may define, for example, aircraft position, velocity, and direction of flight—quantities necessary to solve a specific path planning problem. This paper builds on an open-source 4D gradient-guided contingency landing planner developed in \cite{tekaslan_airspace, tekaslan_search, tekaslan2026airspeed}. It generates dynamically admissible trajectories to reachable landing sites while taking no-fly zones, static airspace corridors, and ground population risk exposure into account. 

Let state $s \in \mathcal{S}$ be a tuple in a state-space $\mathcal{S}\subset \mathbb{R}^4$:
\begin{equation}
    s \triangleq (\varphi, \lambda, h, \chi),
\end{equation}
where $\varphi \in \Phi$ and $\lambda \in \Lambda$ are geodesic coordinates in their respective admissible intervals. Altitude above mean sea level is given by $h \in \mathbb{R}_{\geq 0}$, and direction of flight with respect to North is $\chi \in [0,2\pi)$. A contingency landing path $\mathcal{P} \in \Pi$ is expressed as a sequenced set of discrete states $\mathcal{P} = \left(s_0,s_1,\ldots,s_g\right)$ where $s_0 \in \mathcal{S}$ and $s_g \in \mathcal{S}$ are respectively the initial emergency and goal states. Hence, $\Pi$ denotes the finite set state sequences over $\mathcal{S}$
\begin{equation}
    \Pi \subset \bigcup_{N=0}^{\infty} \mathcal{S}^{N},
\end{equation}
where $N \in \mathbb{Z}^+$ is the number of discrete states. State transitions are achieved with admissible actions. Let $a \in \mathcal{A} \subset \mathbb{R}^4$ be an action—an aircraft maneuver primitive:
\begin{equation}
    a \triangleq (\Delta\chi, \ell, \gamma, \gamma_t),
\end{equation}
where $\Delta \chi \in [\Delta\chi_{\min}, \Delta\chi_{\max}]$ is course change with scalar bounds, $\ell \in \mathbb{R^+}$ is horizontal distance between states, and $\gamma, \gamma_t \in \mathbb{R}$ are flight path angles respectively for straight and turning flights. Kinematic feasibility is satisfied with a coupling of $\Delta \chi$ and $\ell$ \cite{tekaslan2026feasibility}. Airspeed is incorporated through $\gamma_{\cdot}$ by accounting for the airspeed dynamics:
\begin{equation}
    m\frac{d}{dt}v_a = -D(v_a,\gamma_{\cdot}) - mg\sin\gamma_{\cdot}.
\end{equation}
Here, $m$ is aircraft mass, $D$ is drag force, and $g$ is gravitational acceleration. Thus, each action $a$ induces airspeed through the selection of pitch attitude. Overall, a transition from a parent state $s$ to a successor state $s^{j}$ is achieved through World Geodetic System 1984 model \cite{wgs84}, $\mathcal{W}:\mathcal{S}\times \mathcal{A}\rightarrow \mathcal{S}$, such that  $s^j = \mathcal{W}(s,a^j)$ per Figure \ref{fig:state_exp}, $j \in \mathbb{Z}^+$ is the branching factor.
\begin{figure}[t!]
    \centering    \includegraphics[width=.75\linewidth]{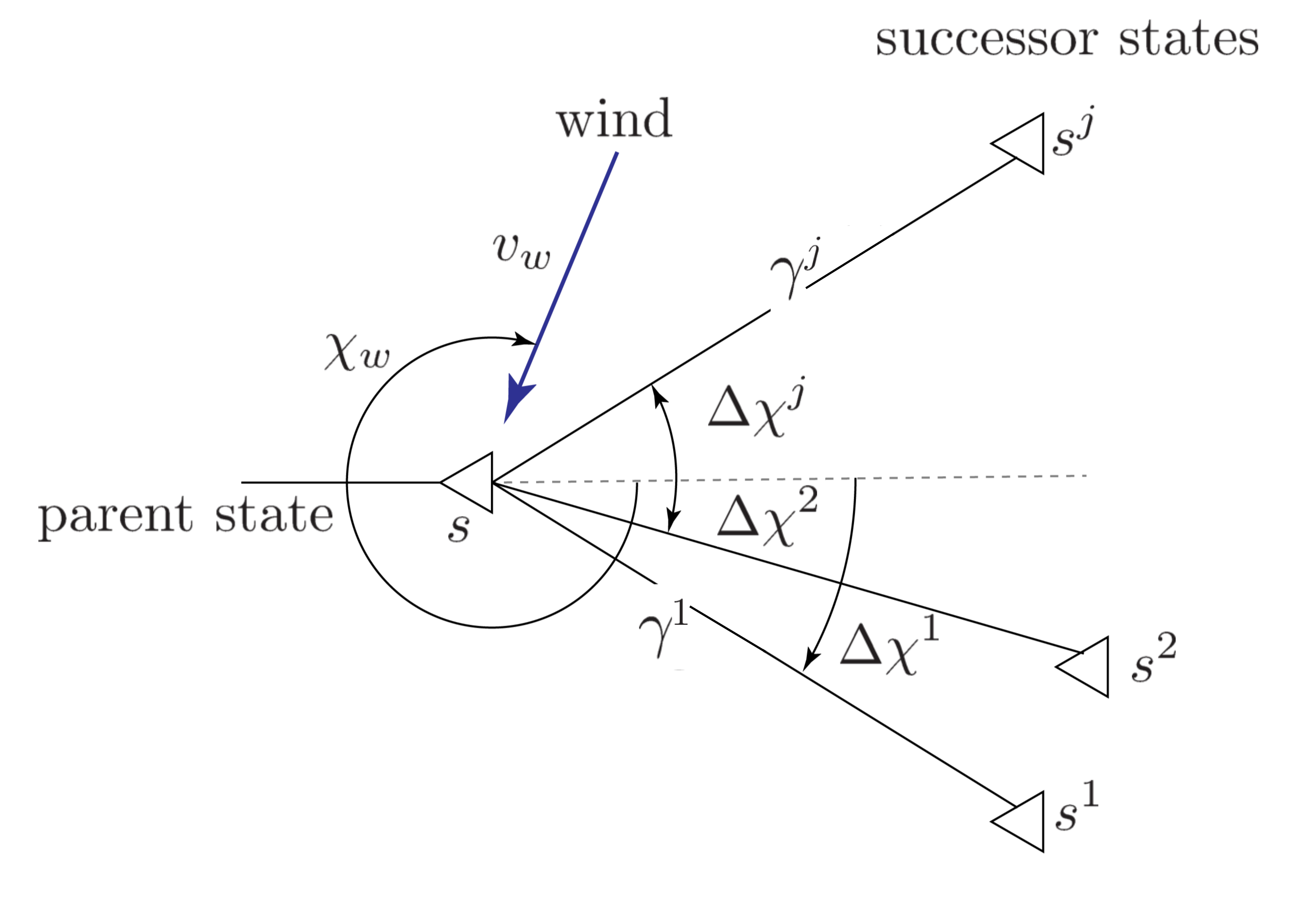}
    \caption{A representative aircraft state expansion on the horizontal plane.}
    \label{fig:state_exp}
\end{figure}

The planner operates using a priority queue (i.e., min-heap) in which the state with the minimum cost based on a cost function $f:\mathcal{S}\rightarrow\mathbb{R}_{\geq0}$ in the heap is expanded next until convergence criteria are satisfied. The cost function is designed for complete loss-of-thrust contingency landing planning, where a wing-lift aircraft is expected to glide until touchdown without thrust force. It is defined as a weighted combination of several terms that embed, for example, optimal glide behavior and landing direction through continuous costs. Additional discrete costs guide the planner away from no-fly zones, densely populated urban regions, and predefined motion corridors with elevated likelihood of multi-agent activity. The explicit expression of $f$ can be found in \cite{tekaslan_airspace}.

The contingency planning problem is defined by the tuple
\begin{equation}
    p \triangleq 
    \left(\mathcal{A}, s_0, s_g, f, v_w, \chi_w\right) \in P,
\end{equation}
with $P$ referring to problem instance space. Wind speed $v_w \in \mathbb{R}$ and direction $\chi_w \in [0, 2\pi)$ are internally used to choose optimal flight path angles for a given state transition to achieve forward-invariant airspeed within safe envelope \cite{tekaslan2026airspeed}. This coupling makes the planner dynamically informed rather than purely kinematic. The tree search operator $\mathcal{T}: P \rightarrow \Pi$ maps this planning problem from an input space $P$ to a discrete path such that $\mathcal{P} = \mathcal{T}(p)$.

\subsection{Well-clear Separation}
Aircraft separation has been studied extensively in several contexts, including air traffic management and Detect-and-Avoid. This study adopts a decoupled separation concept in the horizontal and vertical planes, similar to the studies in \cite{Cook2015, Park2014, Lee2021}, where LoWC occurs only when both horizontal and vertical separation requirements are violated simultaneously.

\begin{figure}[th!]
    \centering
    \includegraphics[width=.5\linewidth]{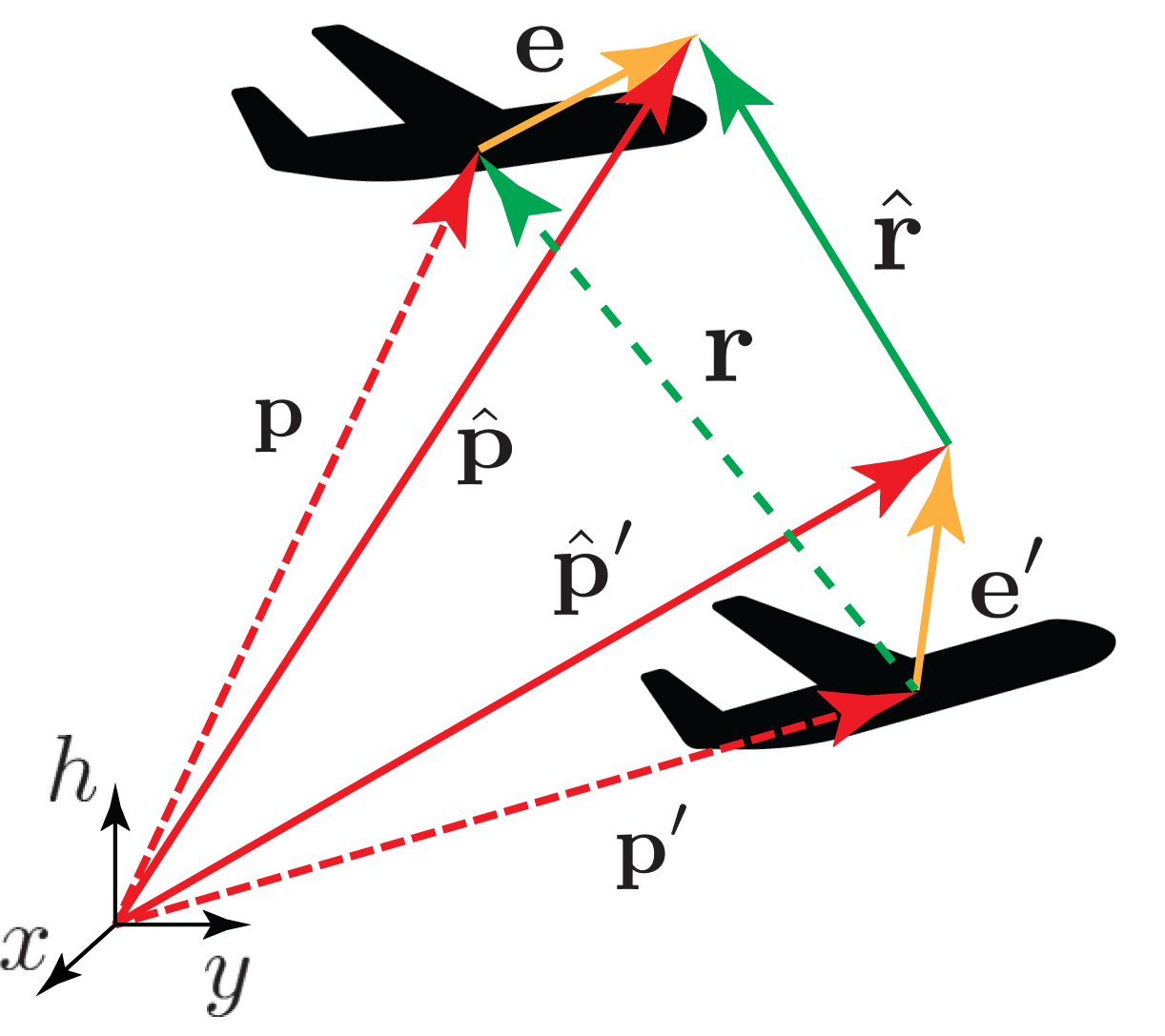}
    \caption{Aircraft pair geometry in a local coordinate frame, showing position vectors, position error vectors, and the resulting relative position vectors.}
    \label{fig:placeholder}
\end{figure}
Let the true position of ego aircraft in a local Euclidean coordinate frame be
\begin{equation}
    \mathbf{p} \triangleq \pi(s) = [x \;\; y \;\; h]^{\intercal},
\end{equation}
where \(x\) and \(y\) are local horizontal coordinates corresponding to latitude \(\varphi\) and longitude \(\lambda\), and \(h\) denotes altitude. The operator \(\pi:\mathcal{S}\rightarrow \mathbb{R}^3\) maps the aircraft state \(s\) to its position in the local coordinate frame. For an ego aircraft and an intruder aircraft, the true relative position vector is defined as
\begin{equation}
    \mathbf{r} = \mathbf{p} - \mathbf{p}^{\prime},
\end{equation}
where $\mathbf{p}^{\prime}$ is the true position of the intruder. Unless otherwise stated, unprimed symbols refer to ego quantities, while primed symbols refer to intruder quantities. The relative position can be decomposed into horizontal and vertical components as
\begin{equation}
    \mathbf{r}_H =
    \begin{bmatrix}
        x - x^{\prime} \\
        y - y^{\prime}
    \end{bmatrix},
    \qquad
    r_V = h - h^{\prime} .
    \label{eq:true_separation}
\end{equation}
\begin{definition}[Well-clear separation]
Given minimum horizontal and vertical separation thresholds $R_H \in \mathbb{R}^{+}$ and $R_V \in \mathbb{R}^{+}$, respectively, an ego--intruder aircraft pair is considered well-clear if the following conditions holds:
\begin{equation}
\lVert \mathbf{r}_H \rVert > R_H \vee |r_V| > R_V .
\end{equation}
Equivalently, LoWC occurs when both separation conditions are violated:
\begin{equation}
\lVert \mathbf{r}_H \rVert \le R_H \wedge |r_V| \le R_V .
\end{equation}
\end{definition}

The definition above assumes exact knowledge of true aircraft positions. In practice, navigation and control errors introduce uncertainty in both ego and intruder positions. Therefore, deterministic well-clear assessment must be extended to account for position uncertainty, which motivates the following section.

%% file: Sections/3_Planning.tex
\section{Contingency Landing Planning in a Multi-Agent Environment}
\label{sec:planning}
This section builds upon the deterministic well-clear separation definition with a probabilistic perspective. Stochastic quantification of position is essential for separation logic in both path planning and conflict resolution.

\subsection{Separation Assurance Under Position Uncertainty}
The purpose of this section is to convert stochastic position uncertainty into deterministic separation buffers that can be evaluated inside a real-time planner with certain confidence level. This work models position error arising from navigation and control systems as a continuous time multi-component stochastic process. Specifically, the error is confidence-bounded, and its horizontal and vertical components admit prescribed probability bounds. The analysis presented in this section does not constrain the results with certain distribution or independence assumptions. In other words, position error can arise from any distribution and be dependent or independent. 

Let an estimated position at any given time $t$ be $\hat{\mathbf{p}} \in \mathbb{R}^3$ such that
\begin{equation}
\hat{\mathbf{p}} = \mathbf{p} + \mathbf{e},
\label{eq:position_w_error}
\end{equation}
where $\mathbf{e} = [e_x \; e_y \; e_h]^{\intercal}$ is the position error vector. It is worth noting that time dependency of position and error is omitted for convenience. The horizontal and vertical error components are defined as
\begin{equation}
    \mathbf{e}_H \triangleq
    \begin{bmatrix}
        e_x\\
        e_y
    \end{bmatrix},
    \qquad
    e_V \triangleq e_h.
\end{equation}
\begin{definition}[Confidence-bounded position error]
For a prescribed horizontal confidence bound $\bar e_H \in \mathbb{R}_{\ge 0}$, the probability that the position error exceeds this bound satisfies
\begin{equation}
\mathbb{P}(\lVert \mathbf{e}_H \rVert > \bar e_H) \le p_H ,
\end{equation}
where $p_H \in [0,1)$ is the probability of horizontal error-bound violation. Similarly, for a vertical confidence bound $\bar e_V \in \mathbb{R}_{\ge 0}$,
\begin{equation}
\mathbb{P}(|e_V| > \bar e_V) \le p_V ,
\end{equation}
where $p_V \in [0,1)$ is the probability of error-bound violation.
\end{definition}
The violation probabilities $p_H$ and $p_V$ are difficult to specify analytically because the total position error is induced by several contributing sources, including state estimation and tracking errors. These error sources may be correlated, time-varying, and dependent on the operating condition. Therefore, this work treats $p_H$ and $p_V$ as design parameters that can be calibrated from empirical error statistics or selected conservatively according to the desired confidence level.

\begin{proposition}[Sufficient condition of well-clear separation under uncertainty]
\label{prop:1}
Let the true relative position between the ego and intruder aircraft be
\begin{equation}
    \mathbf{r}=\hat{\mathbf{r}}-\Delta\mathbf{e},
    \label{eq:true_rel_pos}
\end{equation}
where $\hat{\mathbf{r}}$ is the estimated relative position vector and
$\Delta\mathbf{e} \triangleq \mathbf{e} - \mathbf{e}^{\prime}$ is the relative position error with horizontal $\Delta\mathbf{e}_H$ and vertical $\Delta e_V$ components. If
\begin{equation}
    \lVert \hat{\mathbf{r}}_H \rVert >
    R_H + \lVert \Delta\mathbf{e}_H \rVert
    \quad \lor \quad
    \lvert \hat r_V \rvert
    >
    R_V + \lvert \Delta e_V \rvert,
    \label{eq:well_clear_uncertain}
\end{equation}
then ego and intruder aircraft are well-clear, i.e.,
\begin{equation}
    \lVert \mathbf{r}_H \rVert > R_H
    \quad \lor \quad
    \lvert r_V \rvert > R_V .
\end{equation}
\end{proposition}

\begin{proof}
Since $
    \mathbf{r}_H = \hat{\mathbf{r}}_H-\Delta\mathbf{e}_H,$
the reverse triangle inequality gives with an assumption of $\lVert \mathbf{r}_H \rVert \geq \lVert \Delta \mathbf{e}_H \rvert$,
\begin{equation}
    \lVert \mathbf{r}_H \rVert
    =
    \lVert \hat{\mathbf{r}}_H-\Delta\mathbf{e}_H \rVert
    \geq
    \lVert \hat{\mathbf{r}}_H \rVert
    -
    \lVert \Delta\mathbf{e}_H \rVert .
\end{equation}
Therefore, if $
    \lVert \hat{\mathbf{r}}_H \rVert
    >
    R_H+\lVert \Delta\mathbf{e}_H \rVert,
$
then
$
    \lVert \mathbf{r}_H \rVert > R_H.
$
Similarly, since
$
    r_V=\hat r_V-\Delta e_V,
$
the reverse triangle inequality gives
\begin{equation}
    \lvert r_V \rvert
    =
    \lvert \hat r_V-\Delta e_V \rvert
    \geq
    \lvert \hat r_V \rvert
    -
    \lvert \Delta e_V \rvert .
\end{equation}
Thus, if
$
    \lvert \hat r_V \rvert
    >
    R_V+\lvert \Delta e_V \rvert,
$
then
$
    \lvert r_V \rvert > R_V.
$
Hence, satisfying either inflated horizontal or inflated vertical separation is sufficient to guarantee the corresponding true separation condition for any relative position error satisfying the assumed bounds. Therefore, the aircraft pair is well-clear.
\end{proof}

\begin{proposition} [Probability bound on separation violation under inflated separation]
\label{prop:2}
Let the individual ego and intruder horizontal position errors satisfy
\begin{equation}
\mathbb{P}\left(\lVert \mathbf{e}_H \rVert > \bar e_H\right) \leq p_H,
\qquad
\mathbb{P}\left(\lVert \mathbf{e}_{H}^{\prime} \rVert > \bar e_{H}^{\prime}\right) \leq p_H^{\prime}.
\end{equation}
If the estimated horizontal separation is enforced with the inflated bound
$
\lVert \hat{\mathbf{r}}_H\rVert > R_H + \Delta\bar{e}_H,
$
where
$
\Delta\bar{e}_H \triangleq \bar e_H + \bar e_H^{\prime},
$
then the probability of violating the true horizontal separation threshold is upper-bounded by
\begin{equation}
\mathbb{P}\left(
\lVert \mathbf{r}_H \rVert \leq R_H ;
\lVert \hat{\mathbf{r}}_H\rVert > R_H + \Delta\bar{e}_H
\right)
\leq p_H + p_H^{\prime}.
\end{equation}
Similarly, for the vertical component,
\begin{equation}
\mathbb{P}\left(
\lvert r_V \rvert \leq R_V ;
\lvert \hat{r}_V\rvert > R_V + \Delta\bar{e}_V
\right)
\leq p_V + p_V^{\prime}.
\end{equation}
\end{proposition}

\begin{proof}
First, the triangle inequality provides a conservative confidence bound for the relative horizontal position error:
\begin{equation}
\begin{aligned}
\lVert \Delta \mathbf{e}_H \rVert
&= \lVert \mathbf{e}_H - \mathbf{e}_{H}^{\prime} \rVert \\
&\leq \lVert \mathbf{e}_H \rVert
+ \lVert \mathbf{e}_{H}^{\prime} \rVert \\
&\leq \bar e_H + \bar e_{H}^{\prime}
= \Delta \bar e_H .
\end{aligned}
\end{equation}
Therefore, boundedness of the individual ego and intruder position errors implies boundedness of the relative error:
\begin{equation}
  \lVert \mathbf{e}_H \rVert \leq \bar e_H
  \land
  \lVert \mathbf{e}_{H}^{\prime} \rVert \leq \bar e_{H}^{\prime}
  \implies
  \lVert \Delta \mathbf{e}_H \rVert \leq \Delta \bar e_H .
  \end{equation}
  Using contraposition and De Morgan's law, one obtains
  \begin{equation}
  \lVert \Delta \mathbf{e}_H \rVert > \Delta \bar e_H
  \implies
  \lVert \mathbf{e}_H \rVert > \bar e_H
  \lor
  \lVert \mathbf{e}_{H}^{\prime} \rVert > \bar e_{H}^{\prime}.
  \label{eq:deh_exceeds}
  \end{equation}
  Eq. \eqref{eq:deh_exceeds} provides a conservative upper bound on the probability of relative-error-bound violation. By the union bound,
  \begin{equation}
  \begin{aligned}
  \mathbb{P}\left(
  \lVert \Delta \mathbf{e}_H \rVert > \Delta \bar e_H
  \right)
  &\leq
  \mathbb{P}\left(
  \lVert \mathbf{e}_H \rVert > \bar e_H
  \lor
  \lVert \mathbf{e}_{H}^{\prime} \rVert > \bar e_{H}^{\prime}
  \right)\\
  &\leq
  \mathbb{P}\left(
  \lVert \mathbf{e}_H \rVert > \bar e_H
  \right)
  +
  \mathbb{P}\left(
  \lVert \mathbf{e}_{H}^{\prime} \rVert > \bar e_{H}^{\prime}
  \right)\\
  & =
  p_H+p_{H}^{\prime}.
  \end{aligned}
  \label{eq:prob_eh_greater}
  \end{equation}
Recall from Proposition \ref{prop:1} that
\begin{equation}
\lVert \hat{\mathbf{r}}_H\rVert > R_H + \Delta\bar{e}_H
\land
\lVert \Delta\mathbf{e}_H\rVert \leq \Delta\bar{e}_H
\implies
\lVert \mathbf{r}_H \rVert > R_H .
\end{equation}
Applying contraposition and De Morgan's law once again,
\begin{equation}
\lVert \mathbf{r}_H \rVert \leq R_H
\implies
\lVert \hat{\mathbf{r}}_H\rVert \leq R_H + \Delta\bar{e}_H
\lor
\lVert \Delta\mathbf{e}_H\rVert > \Delta\bar{e}_H .
\end{equation}
Under the enforced inflated-separation condition
$\lVert \hat{\mathbf{r}}_H\rVert > R_H + \Delta\bar{e}_H$, the first condition on the right-hand side cannot hold. Hence, loss of true horizontal separation can occur only if the relative horizontal error exceeds its prescribed bound:
\begin{equation}
\begin{aligned}
\mathbb{P}\left(
\lVert \mathbf{r}_H \rVert \leq R_H ;
\lVert \hat{\mathbf{r}}_H\rVert > R_H + \Delta\bar{e}_H
\right)
&\leq
\mathbb{P}\left(
\lVert \Delta \mathbf{e}_H \rVert > \Delta \bar e_H
\right)\\
&= p_H+p_H^{\prime}.
\end{aligned}
\end{equation}
The same argument applies to the vertical separation, yielding
\begin{equation}
\mathbb{P}\left(
\lvert r_V \rvert \leq R_V ;
\lvert \hat{r}_V\rvert > R_V + \Delta\bar{e}_V
\right)
\leq p_V + p_V^{\prime}.
\end{equation}
\end{proof}

Practically, Proposition \ref{prop:2} allows the planner and advisory generator to use the same inflated well-clear logic without assuming a specific error distribution. If the actual relative error remains within this buffer, then true separation is guaranteed. Otherwise, the only way the aircraft can lose true separation despite satisfying the inflated separation constraint is if the relative position error exceeds its prescribed confidence bound. Thus, the probability of loss of separation is upper-bounded by the probability of relative error bound violation.

\subsection{Real-time Path Planning in a Multi-Agent Environment}
Propositions \ref{prop:1} and \ref{prop:2} have useful implications for real-time contingency planning in a cooperative multi-agent environment. First, they enable enforcement of constant separation buffers using estimated relative positions while providing a confidence guarantee on true separation. Second, they convert position uncertainty into deterministic inflated separation constraints that can be checked efficiently inside a real-time planner. Third, the result is conservative and distribution-free, requiring only confidence bounds on individual position errors rather than probability models or independence assumptions.

Recall from Section \ref{sec:search_prelim}, state expansion in discrete search is determined by a cost function $f(s)$. Search methods with optimality guarantee such as A$^\star$ employs  cumulative cost $g(s)$ and an admissible heuristic cost-to-go $h(s)$ with total cost $f(s) = g(s) + h(s)$. The planner used in this study also calculates the total cost of a state with the same approach whereas $h(s)$ is inadmissible by design to achieve the real-time performance required in contingency planning. The direct implication of sub-optimality within the context of contingency path planning with collision avoidance is that the resulting path would reduce the number of encounters and duration of LoWC but does not guarantee a conflict-free solution, if it exists. Aside from an inadmissible heuristic function, degraded maneuverability may also result in infeasible solutions state-space exploration is inherently degraded. This limitation is addressed in the following section with a complementary deconfliction method yet, first, a cumulative path cost based on Proposition \ref{prop:1}1 is introduced for real-time conflict-aware contingency landing trajectory generation.

Dynamic obstacle avoidance in planning introduces a temporal component to the solution. Thus, the aircraft state introduced in \ref{sec:search_prelim} is revised to avoid additional notation to include timestamp
\begin{equation}
    s = (\varphi, \lambda, h, \chi, t),
    \label{eq:state}
\end{equation}
enabling spatiotemporal planning with $s \in \mathcal{S} \subset \mathbb{R}^5$. A contingency landing path $\mathcal{P}$ is interpreted as a continuous time-parametrized trajectory $\tau: [0,T] \rightarrow \mathcal{S}$ such that $\tau(t) = s$ is the state corresponding to the time instance in planning horizon $[0,T]$ for conflict prediction and separation assessment.

The ego and intruder trajectories are denoted by $\tau,\tau'\in\mathbb{T}$, respectively. The notation $\mathbb{T}^{\prime} \subset \mathbb{T}$ denotes the set of nominal intruder trajectories. The corresponding time-synchronized states are $\tau(t) = s,\;\tau^{\prime}(t) = s^{\prime}, \forall t.$ Additionally, vector notations used in the previous sections are extended to indicate explicit time dependency. For example,
\begin{equation}
    \hat{\mathbf{r}}(t) = \pi(s) - \pi(s^{\prime}) = \pi(\tau(t)) - \pi(\tau^{\prime}(t)).
    \label{eq:rel_pos_time}
\end{equation}

\begin{definition}[Uncertainty-deflated separation]
The uncertainty-deflated horizontal and vertical separations are defined as
\begin{equation}
\delta_H(t) \triangleq \max \left(0, \lVert \hat r_H(t)\rVert - \Delta \bar e_H \right),
\end{equation}
and
\begin{equation}
\delta_V(t) \triangleq \max \left(0, |\hat r_V(t)| - \Delta \bar e_V \right).
\end{equation}
These quantities account for the worst-case relative position error within the prescribed confidence bounds and are used to evaluate well-clear separation under position uncertainty.
\end{definition}
The planner evaluates separation using these conservative metrics.
Here, $\delta_H(t)$ and $\delta_V(t)$ are not raw separations; rather, they are uncertainty-deflated separations that account for the worst-case relative position error within the prescribed confidence bounds as illustrated in Figure \ref{fig:separations}.
\begin{figure}[t!]
    \centering
    \includegraphics[width=\linewidth]{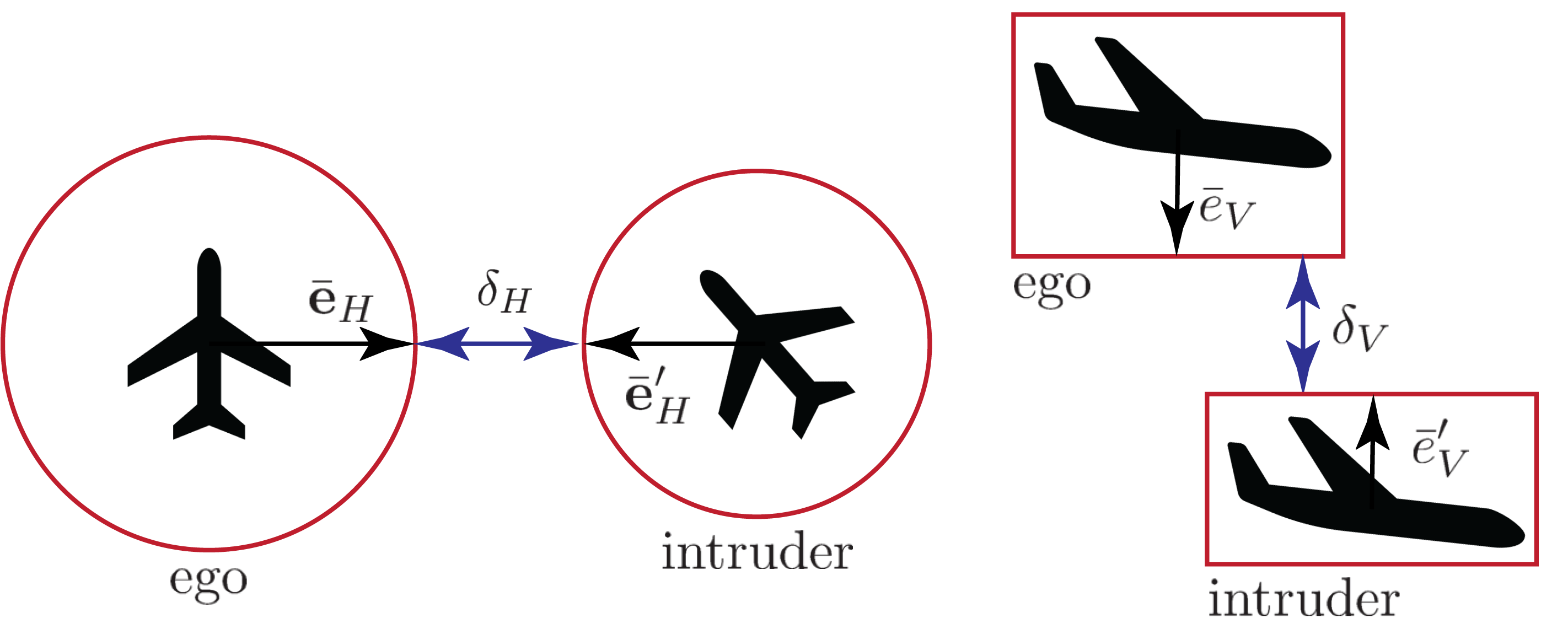}
    \caption{Instantaneous horizontal and vertical separations.}
    \label{fig:separations}
\end{figure}

Using Proposition \ref{prop:1}, the inflated estimated separation condition
can be equivalently checked as
$
    \delta_H(t)>R_H
    \quad \lor \quad
    \delta_V(t)>R_V .
$
Thus, LoWC occurs only when both uncertainty-deflated separations are below their corresponding thresholds
$
    \delta_H(t)\leq R_H
    \quad \land \quad
    \delta_V(t)\leq R_V .
$
\begin{definition}[Loss-of-well-clear set under uncertainty]
The set of state pairs corresponding to LoWC under uncertainty-deflated separation is defined as
\begin{equation}
C_{\mathrm{LoWC}} =
\{(s,s') \mid \delta_H \le R_H \wedge \delta_V \le R_V\}.
\end{equation}
\end{definition}
To penalize such violations during search, horizontal and vertical signed
separation margins are respectively defined as
\begin{equation}
\begin{aligned}
    m_H(t)
    &\triangleq
    \min\left(
    0,\,
    \delta_H(t)-R_H
    \right),\\
    m_V(t)
    &\triangleq
    \min\left(
    0,\,
    \delta_V(t)-R_V
    \right).
\end{aligned}
\end{equation}
The margins $m_H$ and $m_V$ are zero when the corresponding separation requirement is satisfied and negative when the requirement is violated. The planner assigns a smooth penalty based on the magnitudes of these violations. Normalized penalty functionals for horizontal and vertical LoWC are
\begin{equation}
\begin{aligned}
    J_H(t)
    &\triangleq
    1 - \frac{2}{1+\exp(k_H |m_H(t)|)},\\
    J_V(t)
    &\triangleq
    1 - \frac{2}{1+\exp(k_V |m_V(t)|)},
    \end{aligned}
    \label{eq:Jv}
\end{equation}
where $k_H>0$ and $k_V>0$ control the sensitivity of the cost to horizontal and vertical separation violations. For this study, $k_H = 1$ and $k_V = 0.005$ are used. The resulting cost between the ego trajectory and a given intruder trajectory is
\begin{equation}
    J(t \mid \tau,\tau^{\prime})
    \triangleq
    J_H(t)J_V(t).
    \label{eq:Jcpa}
\end{equation}
This multiplicative structure reflects the well-clear definition. LoWC occurs only when both horizontal and vertical separation requirements are violated. Therefore, if either component returns to a well-clear state, the corresponding margin becomes zero and $J$ vanishes. The explicit dependence on the trajectory pair $(\tau,\tau^{\prime})$ in Eqs. \eqref{eq:rel_pos_time}--\eqref{eq:Jv} is omitted for convenience. Eventually, the cumulative loss of separation cost $g:\mathcal{S}\rightarrow\mathbb{R}_{\geq 0}$ is defined as the time integral of the maximum cost over all intruders along the candidate contingency landing trajectory terminating at state $s$:
\begin{equation}
    g(s) \triangleq \int_{0}^{t_s} \max_{\tau^{\prime} \in \mathbb{T}^{\prime}} J(t \mid \tau, \tau^{\prime})dt,
    \label{eq:cumul_lowc_risk}
\end{equation}
where $t_s$ denotes time at state $s$.

By definition, evaluating $g(s)$ requires knowledge of future intruder states over the planning horizon. This is reasonable in cooperative multi-agent systems, where agents either exchange planned trajectories directly or submit operational intent to a shared coordination service \cite{10.1145/3128584}. Therefore, this work assumes that intruder intent, specifically position with timestamps, is available over the finite prediction horizon, with position uncertainty taken into account through the inflated separation bounds derived in the previous section.

\begin{definition}[Conflict set]
When LoWC is predicted, the planner stores conflicting intruder trajectories into $\mathbb{T}^{\prime}_c \subseteq \mathbb{T}^{'}$,
\begin{equation}
\mathbb{T}^{\prime}_c \triangleq \{\tau^{\prime} \in \mathbb{T}^{\prime} \mid \exists t \in [0,T] \text{ such that } (\tau(t), \tau'(t)) \in C_{\mathrm{LoWC}}\}.
\end{equation}
Therefore, conflict encounters of $\tau$ must be resolved $\forall \tau' \in \mathbb{T}^{\prime}_c$.
\end{definition}

%% file: Sections/4_Deconfliction.tex
\section{Conflict Resolution Advisory}
\label{sec:deconfliction}
The conflict resolution advisory algorithm operates on the priority ego landing plan and predicted intruder (cooperative traffic) trajectories. While the planner determines whether an emergency trajectory is feasible or costly with respect to surrounding traffic, the advisory module determines how surrounding traffic paths should be modified in real-time to remove any remaining predicted conflicts. Because the ego aircraft may have degraded maneuverability and/or low energy reserves, its landing plan is prioritized and kept unchanged.

This work proposes a hierarchical unilateral advisory generation framework that modifies intruder trajectories to recover spatiotemporal clearance and maintain well-clear separation. The hierarchy is designed to accommodate both wing-lift and Vertical Take-off and Landing (VTOL) operations across different flight phases, from take-off to landing. The admissible resolution actions are defined as follows:
\begin{enumerate}
    \item \textbf{Halt} advisory, $\mathcal{R}_{\mathrm{Halt}}$, prevents an intruder from taking off when the aircraft is still on the ground at the beginning of advisory generation. This advisory preserves spatial separation by keeping the intruder outside the multi-agent activity and is intended to be issued through datalink communication from the advisory generation module to the affected aircraft.

    \item \textbf{Speed} advisory, $\mathcal{R}_{\mathrm{Speed}}$, modifies the intruder's nominal ground speed profile to achieve temporal separation while preserving the nominal route geometry.

    \item \textbf{Altitude} advisory, $\mathcal{R}_{\mathrm{Alt}}$, modifies the intruder's nominal altitude profile with vertical offset, thereby achieving spatial separation in the vertical plane.

    \item \textbf{Extend} advisory, $\mathcal{R}_{\mathrm{Extend}}$, modifies a local segment of the intruder's nominal trajectory by increasing path length through a lateral maneuver. This advisory can achieve spatial separation, temporal separation, or both.

    \item \textbf{Divert} advisory, $\mathcal{R}_{\mathrm{Divert}}$, generates a new trajectory from the intruder's current route to a known holding point. This is the most intrusive advisory and is used when local timing, vertical, or lateral modifications are insufficient.
\end{enumerate}
\begin{definition}[Resolution advisory]
The set of admissible resolution actions is defined as
\begin{equation}
\mathcal{R} = \{\mathcal{R}_{\mathrm{Halt}}, \mathcal{R}_{\mathrm{Speed}}, \mathcal{R}_{\mathrm{Alt}}, \mathcal{R}_{\mathrm{Extend}}, \mathcal{R}_{\mathrm{Divert}}\}.
\end{equation}
Each advisory is represented by a trajectory transformation operator $\mathcal{R}_{\cdot}: \mathbb{T}'_c \times \mathbb{R}^{n} \rightarrow \mathbb{T}$ such that
\begin{equation}
\tau'_{\mathrm{adv}} = \mathcal{R}{\cdot}(\tau',\mathbf{x}),
\end{equation}
where $\tau'_{\mathrm{adv}} \in T$ is the advised conflict-free intruder trajectory, and $\mathbf{x} \in \mathbb{R}^{n}$ is the advisory-specific decision vector with dimension $n$ that depends on the advisory type.
\end{definition}

As introduced in the previous section, the time horizon for planning is $[0,T]$ starting from contingency onset at $t = 0$ to ego landing $t = T$. Since resolution advisory generation requires outputs from contingency landing planning, the methods execute sequentially. Thus, advisory time horizon is $t \in (0,T]$ as the advisory onset $t_{\mathrm{on}} > t_0 = 0$. This algorithm resolves conflicts without jeopardizing existing multi-agent activity by requiring minimal changes to nominal intruder trajectories. Hence, all advisories except $\mathcal{R}_{\mathrm{Halt}}$ are formulated as quadratic programming problems. 

Before formally defining individual advisories, two constraint functions are introduced as they are common across the advisory solutions. The first constraint enforces well-clear separation with the distressed ego aircraft. Since the contingency landing trajectory is generated in real time using a nominal time parameterization, the ego may be slightly ahead of or behind the planned trajectory because of speed-tracking error during execution. To account for this factor without propagating the ego dynamics during advisory generation, the ego ground speed component along the planned trajectory is assumed to satisfy $0 < v_e^{-} \leq v_e^0 \leq v_e^{+}$, where $v_e^{-}$ and $v_e^{+}$ denote the lower and upper admissible ego ground speed bounds, respectively, and $v_e^0$ is the nominal speed used to time-parameterize the planned contingency trajectory. The lower along-track ground speed bound is assumed to be strictly positive so that the aircraft progresses monotonically along the planned trajectory. If wind causes the along-track ground speed component to become nonpositive, the nominal trajectory tracking assumption no longer holds and the trajectory is treated as infeasible rather than as backward traversal of the path.

At an advisory time $t$, the possible ego position is represented by a time window $\mathcal{W}_e:[0,T]\rightarrow \mathbb{I}_{[0,T]}$ along the nominal contingency trajectory,
\begin{equation}
\mathcal{W}_e(t) = \left[
\frac{v_e^{-}}{v_e^0}t,\;
\min\left(T,\frac{v_e^{+}}{v_e^0}t\right)
\right],
\label{eq:ego_time_window}
\end{equation}
where $\mathbb{I}_{[0,T]}$ denotes the set of closed intervals contained in [0,T].
The lower bound corresponds to the slowest admissible ego progression along the path, while the upper bound corresponds to the fastest admissible ego progression. Therefore, instead of comparing $\tau(t)$ and $\tau^{\prime}(t)$ at the same nominal time index, the intruder state at time $t$ is checked against all ego states $\tau(\xi)$ with $\xi \in \mathcal{W}_e(t)$. The robust well-clear constraint function $\psi_1:\mathbb{T}\times\mathbb{T}\rightarrow\mathbb{R}$ is then defined as
\begin{equation}
\psi_1(\tau,\tau^{\prime}) = \max_{t\in(0,T]}
\max_{\xi\in\mathcal{W}_e(t)}
\min
\left(
R_H-\delta_H(\xi),
R_V-\delta_V(\xi)
\right),
\label{eq:const_1}
\end{equation}
where $\delta_H(\xi)$ and $\delta_V(\xi)$ are the horizontal and vertical separations between the ego state $\tau(\xi)$ and the intruder state $\tau^{\prime}(t)$ in a time window $\xi\in\mathcal{W}_e(t)$, respectively. The two terms inside the minimum expression represent the instantaneous horizontal and vertical separation deficits. A positive deficit indicates that the uncertainty-deflated separation is below the required threshold. The minimum of the two deficits is positive only during simultaneous horizontal and vertical separation violation. The inner maximization over $\xi\in\mathcal{W}_e(t)$ returns the worst-case separation deficit over the admissible ego timing uncertainty at each advisory time. The outer maximization then returns the largest robust LoWC constraint value over the advisory horizon. Consequently, enforcing
$
\psi_1(\tau,\tau^{\prime}) \leq 0
$
requires well-clear separation to be maintained for all admissible ego timing realizations induced by the bounded speed uncertainty. When $v_e^{-}=v_e^0=v_e^{+}$, the time window in Eq. \eqref{eq:ego_time_window} collapses to $\mathcal{W}_e(t)={t}$, and Eq. \eqref{eq:const_1} reduces to the nominal same-time well-clear constraint.

Unlike the well-clear condition between the distressed ego aircraft and an intruder, the same separation requirement may be intentionally relaxed for nominal multi-agent operations due to confidence in full control authority and requirements of missions performed in close proximity. Therefore, the well-clear metric defined previously is not directly imposed between nominally performing intruders. Instead, the advisory is required to not increase the interaction risk of the modified intruder trajectory with respect to the remaining nominal agents. Let $I_J:\mathbb{T}^{\prime}\rightarrow\mathbb{R}_{\geq 0}$ denote the cumulative interaction-risk functional for an intruder trajectory. For a nominal intruder trajectory $\tau_i^{\prime}$, this functional is defined as
\begin{equation}
    I_J = I_J(\tau_i^{\prime}) = \int_0^T \max_{\tau_j^{\prime}\in\mathbb{T}^{\prime},\,j\neq i}
    J\left(t\mid \tau_i^{\prime},\tau_j^{\prime}\right)
    dt.
\end{equation}
After the advisory, the corresponding cumulative value is
\begin{equation}
    I_J^{\mathrm{adv}} = I_J(\tau_{i,\mathrm{adv}}^{\prime}).
\end{equation}
Similarly, let $J_{\max}:\mathbb{T}^{\prime}\rightarrow\mathbb{R}_{\geq 0}$ denote the worst-case instantaneous interaction-risk functional. For a nominal intruder trajectory $\tau_i^{\prime}$, it is defined as
\begin{equation}
    J_{\max} = J_{\max}(\tau_i^{\prime})
    =
    \max_{t\in(0,T]}
    \max_{\tau_j^{\prime}\in\mathbb{T}^{\prime},\,j\neq i}
    J\left(t\mid \tau_i^{\prime},\tau_j^{\prime}\right).
\end{equation}
After applying an advisory, the corresponding worst-case instantaneous
interaction risk is
\begin{equation}
    J_{\max}^{\mathrm{adv}}
    =
    J_{\max}(\tau_{i,\mathrm{adv}}^{\prime}).
\end{equation}
Then, the interaction constraint $\psi_2:\mathbb{T}^{\prime} \times \mathbb{T}^{\prime} \rightarrow \mathbb{R}$ is
\begin{equation}
    \psi_2(\tau_{i}^{\prime}, \tau_{i,\mathrm{adv}}^{\prime}) =
    \max\left( I_J^{\mathrm{adv}} - I_J, J_{\max}^{\mathrm{adv}} - J_{\max} \right).
\end{equation}
Thus, enforcing $ \psi_2(\tau_{i}^{\prime}, \tau_{i,\mathrm{adv}}^{\prime}) \leq 0$ is equivalent to requiring
\begin{equation}
    I_J(\tau_{i,\mathrm{adv}}^{\prime}) \leq I_J(\tau_i^{\prime}),
    \quad
    J_{\max}(\tau_{i,\mathrm{adv}}^{\prime}) \leq J_{\max}(\tau_i^{\prime}).
\end{equation}
The first inequality prevents the advisory from increasing the cumulative interaction risk over the advisory horizon, whereas the second prevents the advisory from increasing the worst-case instantaneous interaction risk. Hence, the modified intruder trajectory is allowed only if it resolves the conflict with the ego without increasing the total exposure and the peak interaction risk with the remaining multi-agent activity. 

Furthermore, optimal decision vectors $\mathbf{x}^{\star}$ for minimally intrusive advisories are found by solving
\begin{equation}
\begin{gathered}
   \mathbf{x}^{\star} = \arg\min_{\mathbf{x} \in \mathbb{X}} \mathbf{x}^{\intercal}\mathbf{x}\\
   \mathrm{s.t.}\; \tau^{\prime}_{\mathrm{adv}} = \mathcal{R}_{\cdot}(\tau', \mathbf{x}),\quad \tau' \in \mathbb{T}'_c\\
   \psi_1(\tau,\tau^{\prime}_{\mathrm{adv}}) \leq0,\;\psi_2(\tau^{\prime}, \tau^{\prime}_{\mathrm{adv}}) \leq 0, \quad \tau^{\prime} \in \mathbb{T}^{\prime}\\
\end{gathered}
\label{eq:adv_opt}
\end{equation}
where $\mathbb{X}$ is the admissible decision vector space specified for each advisory. In the following sections, each individual advisory is formally defined.

\subsection{Halt}
The halt advisory $\mathcal{R}_{\mathrm{Halt}}$ is the only resolution action not formulated as an optimization problem. It is available only to intruders that are still on the ground, regardless of aircraft type. Since the intruder has not yet entered the active airspace, $\mathcal{R}_{\mathrm{Halt}}$ delays its departure until the distressed ego no longer conflicts with original intruder intent. Therefore, $\mathcal{R}_{\mathrm{Halt}}$ preserves the nominal path geometry and speed profile of the intruder, but shifts the entire trajectory forward in time.

Let $\Delta t_{\mathrm{Halt}}\in\mathbb{R}^+$ denote the commanded ground hold or delay duration. The halt advisory transforms the nominal intruder trajectory $\tau^{\prime}$ into the advised trajectory $\tau_{\mathrm{adv}}^{\prime}$ by
\begin{equation}
    \tau_{\mathrm{adv}}^{\prime}(t) = \tau^{\prime}(t-\Delta t_{\mathrm{Halt}}), \quad t\in(\Delta t_{\mathrm{Halt}},T],
\end{equation}
with the intruder remaining at its initial ground state during the holding
interval $\tau_{\mathrm{adv}}^{\prime}(t)= \tau_i^{\prime}(0), t\in[0,\Delta t_{\mathrm{Halt}}].$ Equivalently, the transformation can be written compactly as $\tau^{\prime}_{\mathrm{adv}} = \mathcal{R}_{\mathrm{Halt}} \left(\tau',\Delta t_{\mathrm{Halt}}\right)$. The holding duration is selected as the smallest nonnegative time shift that satisfies the well-clear constraint with the distressed ego trajectory. Since the intruder remains on the ground during the advisory, the halt action does not introduce additional airborne interaction risk with the remaining nominal traffic. Thus, unlike the subsequent advisories, it can be evaluated directly by time shifting the nominal trajectory rather than by solving a quadratic program.

\subsection{Speed}
The advisory $\mathcal{R}_{\text{Speed}}$ resolves a predicted conflict by modifying the intruder's speed profile while preserving its path geometry. We first define a saturation operator 
\begin{equation}
    \operatorname{sat}_{[\underline{x},\overline{x}]}(x) = \min \left ( \max(x,\underline{x}), \overline{x} \right ),
\end{equation}
that bounds given scalar input into the interval $[\underline{x},\overline{x}]$, $\underline{x} < \overline{x}$.

Let $\Delta v_{\mathrm{adv}} \in \mathbb{R}$ be the magnitude of a ground speed change advisory. A speed offset in time $\Delta v:(0,T] \rightarrow \mathbb{R}$ to the nominal speed trajectory is achieved by means of ramp-hold-return profile given vehicle-specific acceleration and deceleration:
\begin{equation}
\begin{aligned}
    \Delta v(t)
    =
    \Delta v_{\mathrm{adv}}
    \bigg (
    &\operatorname{sat}_{[0,1]}
    \left(
    \frac{t-t_{\mathrm{on}}}{\Delta t_{\mathrm{acc}}}
    \right)- \\
    &\operatorname{sat}_{[0,1]}
    \left(
    \frac{t-t_{\mathrm{on}}-\Delta t_{\mathrm{acc}}-t_{\mathrm{hold}}}{\Delta t_{\mathrm{dec}}}
    \right)
    \bigg ).
\end{aligned}
\end{equation}
Here, $t_{\mathrm{hold}} \in \mathbb{R}^+$ denotes the duration over which the advised speed offset is maintained. Also, $\Delta t_{\mathrm{acc}}, \Delta t_{\mathrm{dec}} \in \mathbb{R}_{\geq 0}$ are respectively acceleration and deceleration durations necessary to develop $\Delta v_{\mathrm{adv}}$ change in speed. With a constant rate of change of speed assumption:
\begin{equation}
    \Delta t_{\mathrm{acc}} = \frac{\lvert \Delta v_{\mathrm{adv}}\rvert}{a^+}, \quad \Delta t_{\mathrm{dec}} = \frac{\lvert\Delta v_{\mathrm{adv}}\rvert}{a^-},
\end{equation}
where acceleration $a^+ \in \mathbb{R}^+$ and deceleration $a^-\in \mathbb{R}^+$ are agent-specific admissible values. In other words, $\Delta v(t)$ is a function that generates a ramp speed change profile that saturates at $\Delta v_{\mathrm{adv}}$. Therefore, the advised ground speed trajectory is expressed with
\begin{equation}
    v_{g}^{\mathrm{adv}}(t) = v_{g}(t) + \Delta v(t).
    \label{eq:speed_adv}
\end{equation}
where $v_g(t):[0,T] \rightarrow \mathbb{R}_{\geq0}$ is the nominal ground speed.
Eq. \eqref{eq:speed_adv} is constrained to assure feasible airspeed for wing-lift aircraft to hinder stalling and overspeeding. The corresponding airspeed trajectory $v_{a}^{\mathrm{adv}}:(0,T] \rightarrow \mathbb{R}^+$ can be found with a known horizontal wind field
\begin{equation}
    v_{a}^{\mathrm{adv}}(t) = \left \lVert v_{g}^{\mathrm{adv}}(t)
    \begin{bmatrix}
        \sin\chi_i(t)\\
        \cos\chi_i(t)\\
        0
    \end{bmatrix}
    - v_w
    \begin{bmatrix}
        \sin\chi_w\\
        \cos\chi_w\\
        0
    \end{bmatrix}
    \right \rVert,
\end{equation}
to ensure $v_{a}^{\mathrm{adv}}(t) \in \mathbb{V}, \forall t$ is in safe airspeed envelope $\mathbb{V}$. A conservative $\mathbb{V}$ can be set to accommodate unsteady wind such as gust and turbulence. On the other hand,  $\Delta v(t) = - v_g(t), \exists t \in [0,T]$ can be advised to hover VTOL intruders. The decision vector $\mathbf{x} \in \mathbb{R}^3$ for the speed advisory is
\begin{equation}
    \mathbf{x} = [\tilde{t}_{\mathrm{on}} \; \Delta \tilde{v}_{\mathrm{adv}} \; \tilde{t}_{\text{hold}}]^{\intercal} \in \mathbb{R}^3,
\end{equation}
where each decision variable is normalized:
\begin{equation}
    \tilde{t}_{\mathrm{on}} = \frac{t_{\mathrm{on}}}{t_c}, \quad \Delta \tilde{v}_{\mathrm{adv}} = \frac{\lvert\Delta v_{\mathrm{adv}}\rvert}{\Delta v_{\mathrm{adv}}^{\max}}, \quad \tilde{t}_{\text{hold}} = \frac{t_{\text{hold}}}{T}.
\end{equation}
Here, $t_c > t_{\mathrm{on}} \in (0,T]$ denotes the first conflict instance, and the variable with the superscript "max" denotes agent- and case-dependent positive scalar upper bound. Therefore, the optimal solution of Eq.~\eqref{eq:adv_opt} resolves the predicted conflict while minimizing the advisory initiation time through $\tilde{t}_{\mathrm{on}}$ and minimizing deviations from the nominal trajectory through $\Delta\tilde v_{\mathrm{adv}}$ and $\tilde t_{\mathrm{hold}}$, thereby preserving mission fidelity.

\subsection{Altitude}
The altitude advisory $\mathcal{R}_{\mathrm{Alt}}$ resolves a predicted conflict by modifying the intruder's altitude profile while preserving its nominal lateral path and ground speed trajectory. This advisory applies to airborne intruders and is formulated as an optimization problem under the generic advisory structure in Eq.~\eqref{eq:adv_opt}. The purpose of the altitude advisory is to introduce the smallest vertical deviation
necessary to restore well-clear separation from the distressed ego trajectory without increasing the interaction risk with the rest of the agents.

Let $\Delta h_{\mathrm{adv}}\in\mathbb{R}$ denote the advised altitude offset. The corresponding altitude offset trajectory $\Delta h:(0,T]\rightarrow\mathbb{R}$ is generated using a ramp--hold--return
profile:
\begin{equation}
\begin{aligned}
    \Delta h(t) =
    \Delta h_{\mathrm{adv}}
    \bigg (
    &\operatorname{sat}_{[0,1]}
    \left(
    \frac{t-t_{\mathrm{on}}}{\Delta t_{\mathrm{climb}}}
    \right)- \\
    &\operatorname{sat}_{[0,1]}
    \left(
    \frac{t-t_{\mathrm{on}}-\Delta t_{\mathrm{climb}}-t_{\mathrm{hold}}}
    {\Delta t_{\mathrm{desc}}}
    \right)
    \bigg ).
\end{aligned}
\end{equation}
The quantities $\Delta t_{\mathrm{climb}}$ and $\Delta t_{\mathrm{desc}}$ denote the time required to develop and remove the advised altitude offset under constant ground speed at $t= t_{\mathrm{on}}$. They are given by
\begin{equation}
    \Delta t_{\mathrm{climb}} =
    \frac{\lvert \Delta h_{\mathrm{adv}}\rvert}{v_g(t_{\mathrm{on}})\sin\gamma_1},
    \quad
    \Delta t_{\mathrm{desc}} =
    \frac{\lvert \Delta h_{\mathrm{adv}}\rvert}{v_g(t_{\mathrm{on}})\sin\gamma_2},
\end{equation}
where $\gamma_1 \in [\underline{\gamma}_1, \overline{\gamma}_1]$ and 
$\gamma_2 \in [\underline{\gamma}_2, \overline{\gamma}_2] $ 
denote the flight-path angles used to deviate from and recover to the nominal trajectory, respectively. The admissible sets are agent-specific; for instance, VTOL agents may be able to climb or descend at steeper angles than wing-lift aircraft. Overall, the advised altitude trajectory is
\begin{equation}
    h_{\mathrm{adv}}(t) = h(t) + \Delta h(t),
\end{equation}
while $\varphi, \lambda, \chi$, and $v_g$ are inherited from the nominal trajectory. The altitude modification can optionally be constrained above a pre-defined floor altitude to ensure vertical obstacle clearance. The altitude advisory decision vector is
\begin{equation}
    \mathbf{x} = [\tilde{t}_{\mathrm{on}} \; \Delta \tilde{h}_{\mathrm{adv}} \; \tilde{\gamma}_1 \; \tilde{\gamma}_2 \; \tilde{t}_{\mathrm{hold}}]^{\intercal} \in \mathbb{R}^5.
\end{equation}
The normalized variables are
\begin{equation}
    \Delta \tilde{h}_{\mathrm{adv}} = \frac{\lvert\Delta h_{\mathrm{adv}}\rvert}{\Delta h_{\mathrm{adv}}^{\max}},
    \quad
    \tilde{\gamma}_{\cdot} = \frac{\lvert \gamma_{\cdot}\rvert}{\overline{\gamma}_{\cdot} - \underline{\gamma}_{\cdot}}
    .
\end{equation}
Thus, the optimal solution introduces the smallest altitude change through shallow flight path angles, with the earliest necessary advisory initiation and shortest required holding duration, such that $\psi_1(\tau,\tau_{\mathrm{adv}}^{\prime}) \leq 0$ and
$\psi_2(\tau^{\prime},\tau_{\mathrm{adv}}^{\prime}) \leq 0$.

\subsection{Extend}
The extend advisory $\mathcal{R}_{\mathrm{Extend}}$ resolves a predicted
conflict by introducing a geometric extension to the intruder trajectory. The initial and final advisory states are selected on the nominal trajectory, and a new connecting path is generated between them. Therefore, the advised path temporarily deviates from the nominal path and subsequently recovers to it.

Again, $t_{\mathrm{on}}$ denotes the advisory onset time, corresponding to the initial advisory state and $t_{\mathrm{rec}} \in (0,T]$ denotes the recovery time, corresponding to the state on the nominal trajectory at which the advised trajectory rejoins the nominal path,
\begin{equation}
    s_{\mathrm{on}}^{\prime} = \tau^{\prime}(t_{\mathrm{on}}), \quad s_{\mathrm{rec}}^{\prime} = \tau^{\prime}(t_{\mathrm{rec}}).
\end{equation}
The interval of $t_{\mathrm{rec}}$ implies that the final advisory state may lie backward in time compare to initial advisory state. An extension trajectory $\tau_e^{\prime} \in \mathbb{T}$ is computed such that
\begin{equation}
    \tau_e^{\prime}(0) = s_{\mathrm{on}}^{\prime},
    \quad
    \tau_e^{\prime}(t_{\mathrm{rec}}-t_{\mathrm{on}}) = s_{\mathrm{rec}}^{\prime}.
\end{equation}
In this study, an S-turn Dubins path solver~\cite{atkins2006emergency} is chosen to compute $\tau_e^{\prime}$ for its computational efficiency. An example S-turn extension is shown by the green curve in Fig.~\ref{fig:sturn}, while the black curve between the same states represents the nominal trajectory.
\begin{figure}[t!]
    \centering
    \includegraphics[width=\linewidth]{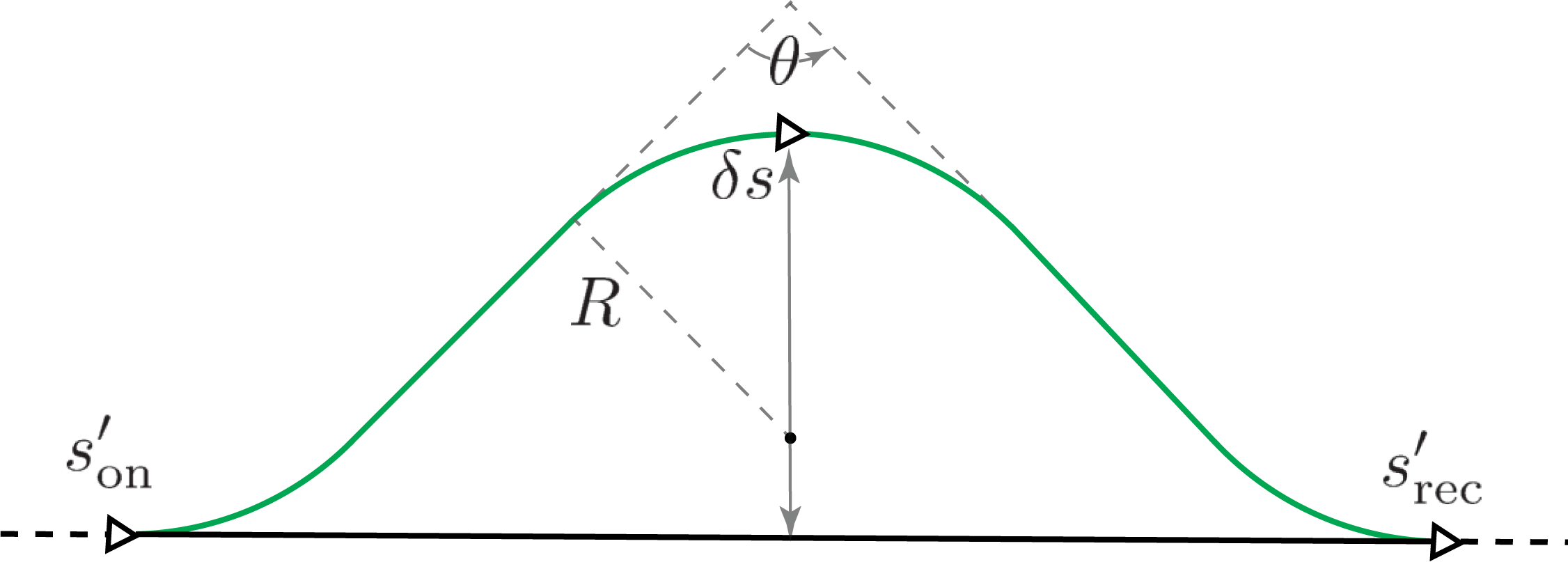}
    \caption{An example S-turn Dubins path geometry on the horizontal plane.}
    \label{fig:sturn}
\end{figure}
An S-turn path is obtained by concatenating two turn-straight-turn Dubins paths \cite{dubins1957curves} with a constant turn radius $R = v_g(t_{\mathrm{on}})/ \dot{\chi}$ with a constant positive turn rate $\dot\chi \in \mathbb{R}^+$.
Let $\mathcal{D}:\mathcal{S}\times\mathcal{S}\rightarrow\mathbb{T}$ denote a Dubins path solver. The S-turn extension is defined as
$
    \tau_e^{\prime} = \mathcal{D}(s_{\mathrm{on}}^{\prime},\delta s) \cup \mathcal{D}(\delta s,s_{\mathrm{rec}}^{\prime}),
$
where $\delta s\in\mathcal{S}$ is an intermediate state used to shape the
extension. It is placed relative to the line segment connecting $s_{\mathrm{on}}^{\prime}$ and $s_{\mathrm{rec}}^{\prime}$. Its lateral displacement from this line controls the extent, while the angle $\theta$ controls the displacement of $\delta s$. When $\theta$ approaches to zero, the straight segments of the two Dubins paths become approximately parallel and extend away from the nominal trajectory, producing a racetrack-like geometry. When $\theta=\pi$, the straight segments tend to overlap with the nominal direction, reducing the lateral extent of the maneuver. However, this limiting case does not necessarily recover the nominal trajectory exactly, because the Dubins solver may still generate orbiting segments with nearly complete circular arcs.

The control angle $\theta$ is parameterized by a centered normalized and continuous decision variable $z_\theta\in[-1,1]$ for optimization and restricted to the disjoint admissible set $\theta \in [\underline{\theta}_{l},\overline{\theta}_{l}] \cup (\underline{\theta}_{r},\overline{\theta}_{r}]$. The value $z_\theta=0$ corresponds to $\theta=\pm 180^\circ$, while the magnitude $|z_\theta|$ moves the advisory toward the admissible bounds on either side. The mapping is
\begin{equation}
\theta(z_\theta)
=
\begin{cases}
\underline{\theta}_{l}
+
|z_\theta|
\left(
\overline{\theta}_{l}-\underline{\theta}_{l}
\right),
& -1 \leq z_\theta < 0,
\\[6pt]
\overline{\theta}_{r}
-
z_\theta
\left(
\overline{\theta}_{r}-\underline{\theta}_{r}
\right),
& 0 \leq z_\theta \leq 1.
\end{cases}
\end{equation}
The first and second branches map $z_\theta$ to the left and right extension directions, respectively.

For altitude matching at the initial and final states of the extension, the altitude profile of $\tau_e^{\prime}$ is projected onto the vertical plane with a constant flight path angle $\gamma$ such that
\begin{equation}
    \gamma = \tan^{-1}\frac{h_{\mathrm{rec}}^{\prime} - h_{\mathrm{on}}^{\prime}}{L} \in [\underline{\gamma}, \overline{\gamma}],
\end{equation}
where $L$ is the extended traversal, and $[\underline{\gamma}, \overline{\gamma}]$ is the admissible range, enforced along with the constraints in Eq. \eqref{eq:adv_opt}. The decision vector for this advisory is
\begin{equation}
    \mathbf{x} = 
    \begin{bmatrix}
        \left (\lvert t_{\mathrm{on}} - t_{\mathrm{rec}}\rvert \right) / {T}\;
        z_{\theta}
    \end{bmatrix}^{\intercal} \in \mathbb{R}^2
\end{equation}
The first element of $\mathbf{x}$ represents the normalized temporal separation between the initial and final advisory states selected on the nominal trajectory. Thus, minimizing $\lvert t_{\mathrm{on}}-t_{\mathrm{rec}}\rvert/T$ encourages the extension to be localized around a short segment of the nominal trajectory. In the limiting case $t_{\mathrm{on}}=t_{\mathrm{rec}}$, the extension starts and ends at the same nominal state, resulting in a localized maneuver rather than a replacement of an extended portion of the nominal path. This formulation is intended to minimize the spatial footprint of the extension thus limit unnecessary interference with surrounding nominal multi-agent activity. It also preserves mission fidelity for operations such as surveillance, inspection, or search-and-rescue, where skipping a portion of the nominal trajectory may lead to loss of coverage or missed task objectives.

\subsection{Divert}
The advisory $\mathcal{R}_{\mathrm{Divert}}$ is the most intrusive
resolution. It generates a new path from the intruder's nominal trajectory to a predefined holding point, thereby diverting the intruder from its original mission.  Let $\mathbb{H}$ be a finite set of holding points
\begin{equation}
    \mathbb{H} = \{q_i \mid q = (\varphi,\lambda,\underline{h},\overline{h},\chi), i \in \mathbb{Z}^+\}.
\end{equation}
Each holding point specifies a horizontal position and inbound course angle, but does not prescribe a fixed altitude. Instead, each holding point is associated with an admissible altitude interval with floor $\underline{h}$ and ceiling $\overline{h}$ altitudes. This freedom gives the optimizer additional flexibility to satisfy well-clear constraints along the divert trajectory. It also allows multiple intruders to be assigned to the same holding point at different altitudes, provided that the resulting trajectories do not increase the interaction risk with the remaining nominal agents.

For a holding point $q \in \mathbb{H}$, the terminal divert state is $s_{h}^{\prime}$ obtained by augmenting $q$ with
a flight path angle such that
\begin{equation}
    h_h^{\prime} = h_{\mathrm{on}}^{\prime} + L\tan\gamma, \quad s_h^{\prime} = (\varphi_h^{\prime}, \lambda_h^{\prime}, h_h^{\prime}, \chi_h^{\prime}),
\end{equation}
and $L$ is the diversion traversal to the holding point. A Dubins-based path is generated from $s_{\mathrm{on}}^{\prime}$ to $s_{h}^{\prime}$ using the Dubins path solver $\mathcal{D}$ such that the advised diversion trajectory is $\tau_{\mathrm{adv}^{\prime}} = \tau^{\prime}(t; t\in [0,t_{\mathrm{on}}]) \cup \mathcal{D}(s_{\mathrm{on}}^{\prime}, s^{\prime})$. The advised trajectory is therefore obtained by replacing the nominal trajectory after $t_{\mathrm{on}}$ with the divert path to the selected holding point. The decision vector is
\begin{equation}
    \mathbf{x} =
    \begin{bmatrix}
        \tilde t_{\mathrm{on}} \; \tilde{\gamma}
    \end{bmatrix}^{\intercal} \in \mathbb{R}^2, \quad \tilde{\gamma} = \max\left(0,
\frac{-\gamma_{\mathrm{thr}}-\gamma}
{-\gamma_{\mathrm{thr}}-\underline{\gamma}},
\frac{\gamma-\gamma_{\mathrm{thr}}}
{\overline{\gamma}-\gamma_{\mathrm{thr}}}
\right).
\end{equation}
The feasible flight path angle range is enforced as in $\mathcal{R}_{\mathrm{Extend}}$, while the preferred range is defined by $\lvert \gamma \rvert \leq \gamma_{\mathrm{thr}} \in \mathbb{R}^+$. This study uses $\gamma_{\mathrm{thr}} = 3^{\circ}$.

Candidate holding points are evaluated according to proximity to $s_{\mathrm{on}}^{\prime}$. Thus, nearby holding points are prioritized to reduce the spatial and temporal extent of the diversion. If no feasible trajectory is found for the closest holding point, the next candidate holding point is evaluated. This process continues until a feasible divert trajectory is found or all candidate holding points have been exhausted. In this way, $\mathcal{R}_{\mathrm{Divert}}$ provides a terminal resolution action when less intrusive advisories are insufficient, while still preserving compatibility with multi-agent traffic through interaction-risk constraint.

%% file: Sections/5_Algorithm.tex
\section{Implementation of the Conflict Resolution Advisory}
\label{sec:alg}
To support real-time deployment, this section presents the proposed algorithmic implementation of the conflict resolution advisory set $\mathcal{R}$, summarized in Algorithm~\ref{alg:parallel_resolution}.
\begin{algorithm}[ht!]
\caption{Conflict Resolution Advisory}
\label{alg:parallel_resolution}
\begin{algorithmic}[1]
\Require Ego trajectory $\tau$, conflict set $\mathbb{T}_{c}'$, intruder trajectories $\mathbb{T}'$, runtime limit $t_{\mathrm{limit}}$
\Ensure Advisory-modified intruder trajectories $\mathbb{T}'$

\ForAll{$\tau' \in \mathbb{T}'_c$ }

    \If{$\tau'(0)$ is on the ground}
        \State $\tau'_{\mathrm{adv}} \gets \mathcal{R}_{\mathrm{Halt}}(\tau',T)$
    \Else
        \State $\mathcal{F} \gets \langle \rangle$ \Comment{Initialize rank-ordered solution set}
        \ForAll{$\mathcal{R}_k \in \mathcal{R}_{\mathrm{Ordered}}$ \textbf{ in parallel}}
        \State $i \gets \operatorname{rank}(\mathcal{R}_k)$
        \State $\mathcal{F}_i \gets \varnothing$
            \State $\tau'_k\gets\textsc{Optimize}(\mathcal{R}_k,\tau',t_{\mathrm{limit}})$
                \If {$\tau_k^{\prime}$ is feasible}
                    \State $\mathcal{F}_i \gets \tau_k^{\prime}$
                \EndIf
        \EndFor

        \If{$\exists i \text{ such that }\mathcal{F}_i\neq \varnothing$}
            \State $i^\star \gets \min\left\{i \mid \mathcal{F}_i \neq \varnothing \right\}$
            \State $\tau_{\mathrm{adv}}' \gets \mathcal{F}_{i^\star}$
        \Else
            \State Mark $\tau'$ as unresolved
            \State \textbf{continue}
        \EndIf
    \EndIf
\State Replace $\tau'$ with $\tau'_{\mathrm{adv}}$ in $\mathbb{T}'$
\EndFor
\State \Return $\mathbb{T}'$
\end{algorithmic}
\end{algorithm}
$\mathcal{R}_{\mathrm{Halt}}$ is always issued to safely delay takeoff when intruder aircraft are still on the ground.  Optimization-based advisories for airborne intruders are evaluated in order of increasing intrusiveness. The rank-ordered advisory set is defined as
\begin{equation}
    \mathcal{R}_{\mathrm{Ordered}} =
    \left\langle
    \mathcal{R}_{\mathrm{Speed}},
    \mathcal{R}_{\mathrm{Alt}},
    \mathcal{R}_{\mathrm{Extend}},
    \mathcal{R}_{\mathrm{Divert}}
    \right\rangle.
\end{equation}
The operator $\operatorname{rank}$ returns the position of an advisory type within $\mathcal{R}_{\mathrm{Ordered}}$, with lower values corresponding to less intrusive advisories, such that $\operatorname{rank}(\mathcal{R}_{\mathrm{Speed}}) = 1$ and $\operatorname{rank}(\mathcal{R}_{\mathrm{Divert}}) = 4$.
Intrusiveness of advisories is assessed based on the changes to the nominal trajectory geometry. For instance, $\mathcal{R}_{\mathrm{Speed}}$ only alters the ground speed profile without affecting the path geometry; therefore, it is deemed the least intrusive method. While $\mathcal{R}_{\mathrm{Alt}}$ only modifies the vertical profile, $\mathcal{R}_{\mathrm{Extend}}$ partly introduces a new local segment with changes in both horizontal and vertical planes; thus, $\mathcal{R}_{\mathrm{Alt}}$ is less intrusive than $\mathcal{R}_{\mathrm{Extend}}$. The most intrusive advisory is $\mathcal{R}_{\mathrm{Divert}}$ as it assigns a completely different path to the intruder. This hierarchy is important. Multiple feasible advisories may be found within a runtime budget $t_{\mathrm{limit}}$; therefore, candidates are selected according to $\mathcal{R}_{\mathrm{Ordered}}$ so that the least intrusive feasible modification is returned. By default, all advisory solvers are executed until the common time limit $t_{\mathrm{limit}}$. The best feasible candidates found within the allowed runtime are stored in a rank-ordered solution set $\mathcal{F}$ as in line~11 in Algorithm~\ref{alg:parallel_resolution}, and the top-ranked feasible advisory candidate in $\mathcal{F}$, corresponding to the least intrusive advisory, is selected per lines~15-16. For time-critical cases such as in-flight emergencies, an optional latency-prioritized early acceptance rule may be used. For instance, if an advisory reaches optimality before $t_{\mathrm{limit}}$, the remaining solvers can be terminated to conserve computational resources, and that advisory is accepted immediately. If no feasible advisory is found within the available advisory set and computation time, the conflict is marked as unresolved. It must then be addressed through a higher-level action, such as replanning the priority trajectory or selecting an alternative candidate. It is worth mentioning that each selected advisory modifies the predicted multi-agent state space over time and therefore affects subsequent conflict resolutions, as reflected in line~22.

In practice, the trajectories are discrete, sampled at a finite frequency, and the advisory transformations include non-differentiable operators such as saturation. To address, a derivative-free global optimizer with multi-start is employed. Also, the completeness of Algorithm~\ref{alg:parallel_resolution} is advisory-relative rather than global. Specifically, the algorithm is complete with respect to the predefined advisory set $\mathcal{R}_{\mathrm{Ordered}}$ if the corresponding optimization routine can recover a feasible solution. For a conflicting intruder trajectory $\tau'\in \mathbb{T}_c'$, if $\exists \mathbf{x} \in \mathbb{X}$ such that $\psi_1(\tau,\mathcal{R}_{\cdot}(\tau',\mathbf{x}))\leq 0$ and $\psi_2(\tau',\mathcal{R}_{\cdot}(\tau',\mathbf{x}))\leq 0$, then the optimizer finds it. However, restricted computational resources prevent completeness guarantees under fixed runtime limits. 

%% file: Sections/6_Application.tex
\section{Applications}
\label{sec:application}
The following sections evaluate the proposed framework as a coupled real-time contingency landing planning and deconfliction architecture. First, synthetic validation cases are presented to demonstrate the planner response to static and moving obstacle encounters. The planner is then benchmarked against Dubins-based three-dimensional trajectories, where all candidate trajectories are generated and the trajectory with the minimum cumulative LoWC risk, defined in \eqref{eq:cumul_lowc_risk}, is selected for comparison. For benchmarking, noisy real-world ADS-B-based trajectory data from the Washington, D.C., airspace are used as dynamic-agent motion traces. Although the data originate from conventional air traffic surveillance, they provide a practical proxy for aerial robotics. The trajectories are openly available \cite{opensky}, noisy, time-varying, and far more cost- and time-efficient to obtain than comparable multi-agent flight experiments. They also capture dynamic agent motion in low-altitude urban airspace, which is a relevant operating environment for AAM and small UAS. 

The ego platform is modeled as an AAM-sized wing-lift aircraft. The dynamic agents include AAM-sized wing-lift and VTOL aircraft comparable in size, as well as commercial transport aircraft which is referred as airliner. All trajectories are sampled at a constant rate of 1~Hz. The path planning and deconfliction modules are implemented in C/C++ and executed on a personal computer with a 3.49~GHz Apple M2 chip.

Vehicle performance parameters used in this study are summarized in Table \ref{tab:perf_params}. The ego ground speed interval is taken from previous work that reports the realized airspeed of the same wing-lift aircraft model under unsteady wind \cite{tekaslan2026airspeed}. Position error bounds are adopted from Federal Aviation Administration avionics system certification requirements \cite{FAA_AC_20_138D_2016, FAA_AC_20_165B_2015}. The horizontal and vertical error bounds are modeled as piecewise functions above and below 2,500 ft and 5,000 ft, respectively. For example, the vertical position error bound is set to 210 ft above 5,000 ft and 160 ft below 5,000 ft.
\begin{table}[t!]
    \centering
    \caption{Vehicle performance parameters.}
    \renewcommand{\arraystretch}{1.3}
    \begin{tabular}{llc}
    \hline \hline
        \textbf{Type} & \textbf{Parameter} & \textbf{Value}\\
        \hline
        \multirow{3}{*}{Ego} & $v_e^-$, $v_e^0$ $v_e^+$ [kts] & 85, 90, 95\\
        & $\bar{\mathbf{e}}_H$ [ft] & 303.8\\
        & $\bar{e}_V$ [ft] & \{160, 210\}\\
        \hline
        \multirow{4}{*}{All intruders} & $\bar{\mathbf{e}}_H$ [NM] & \{0.3, 1\}\\
        & $\bar{e}_V$ [ft] & \{160, 210\}\\
        & $a^+, a^-$ [kts/s] & 1\\ 
        & $\Delta h_{\mathrm{adv}}^{\max}$ [ft] & 1000\\
        \hline
        \multirow{2}{*}{VTOL} & $\Delta v_{\mathrm{adv}}$ [kts] & [-150 150]\\
        & $[\underline{\gamma}_1, \overline{\gamma}_1]$, $[\underline{\gamma}_2, \overline{\gamma}_2]$ [$^{\circ}$] & [-15 15]\\ 
        \hline
        \multirow{3}{*}{AAM wing-lift} 
        & $\Delta v_{\mathrm{adv}}$ [kts] & [-30 30]\\
        & $\mathbb{V}$ [kts] & [70 150]\\
        & $[\underline{\gamma}_1, \overline{\gamma}_1]$, $[\underline{\gamma}_2, \overline{\gamma}_2]$ [$^{\circ}$] & [-5 3]\\ 
        \hline
        \multirow{3}{*}{Airliner} 
        & $\Delta v_{\mathrm{adv}}$ [kts] & [-30 30]\\
        & $\mathbb{V}$ [kts] & [150 350]\\
        & $[\underline{\gamma}_1, \overline{\gamma}_1]$, $[\underline{\gamma}_2, \overline{\gamma}_2]$ [$^{\circ}$] & [-5 3]\\ 
        \hline
    \end{tabular}
    \label{tab:perf_params}
\end{table}

\subsection{Dynamic Multi-agent Simulations}
Dynamic simulations are used to evaluate the end-to-end behavior of the proposed contingency planning and deconfliction framework beyond kinematic trajectory-level analysis. The ego vehicle is modeled as an engine-out Cessna 182, as in \cite{tekaslan_airspace}, representing a wing-lift AAM-sized aircraft undergoing a loss-of-thrust and executing the generated emergency landing plan. Nearby traffic is modeled using dynamic aircraft models, including a full-size commercial airliner based on the NASA Generic Transport Model (GTM) \cite{Jordan2008, Tekles2017}, to evaluate how intruders respond to issued resolution advisories. All aircraft models are represented with six-degree-of-freedom dynamics.

The resulting simulation environment integrates contingency landing path planning, conflict prediction, advisory generation, and closed-loop aircraft response in a unified multi-agent setting. The dynamic simulations are only used to assess whether advisories produced by the kinematic PAA pipeline remain dynamically realizable when executed by representative aircraft models. This enables the advised and dynamically realized trajectories to be compared directly, providing a higher-fidelity assessment of both well-clear recovery and the dynamic consistency of the coordinated response.

\subsection{Conflict-Aware Contingency Landing Planning Validation}
Two synthetic conflict scenarios are designed to validate the discrete search-based conflict-aware contingency trajectory generation. The first scenario is demonstrated in Fig. \ref{fig:validation1}.
\begin{figure}[t!]
    \centering
    \includegraphics[width=\linewidth]{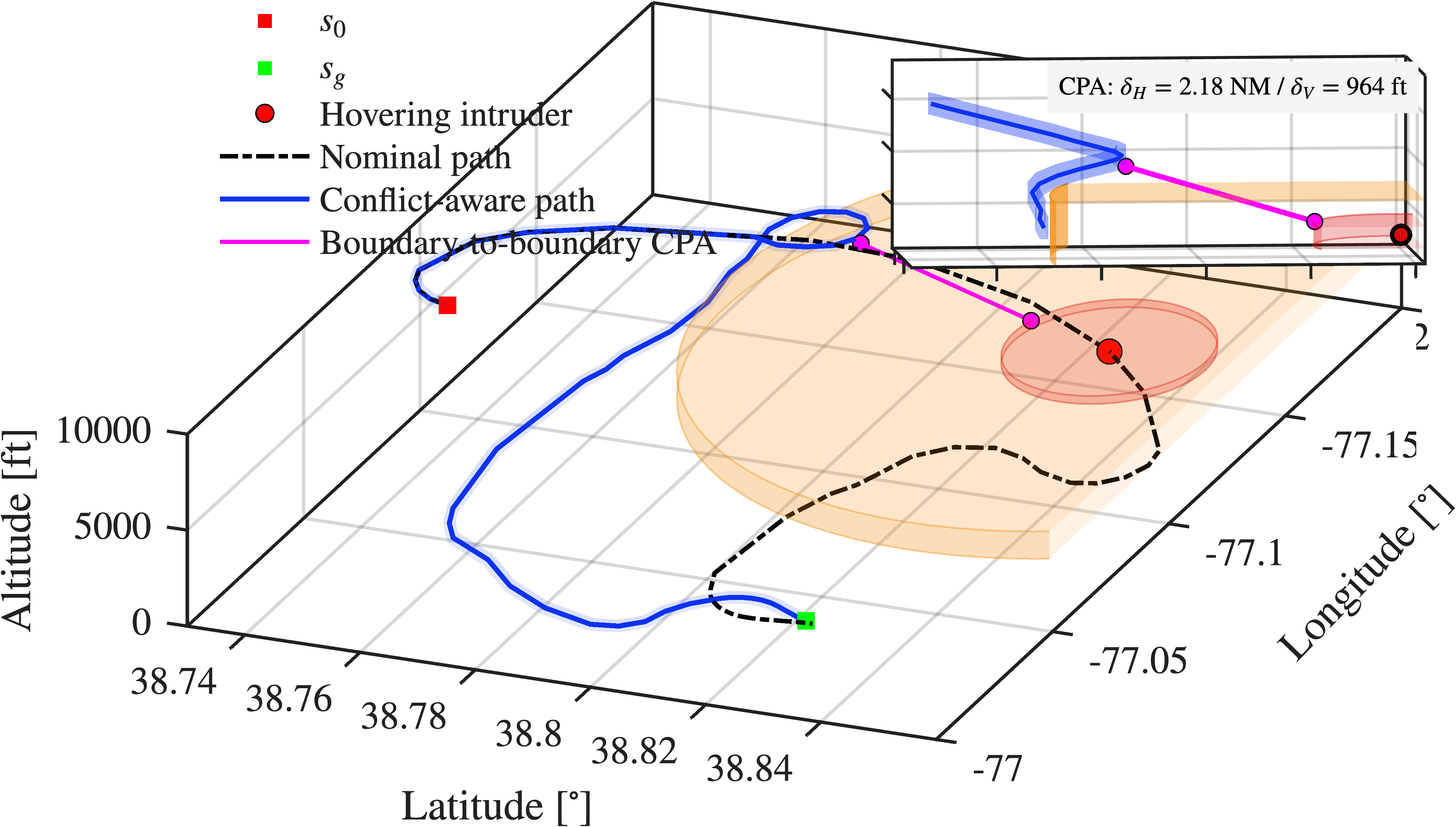}
    \caption{A validation case with a stationary VTOL intruder. The red cylinder represents the confidence-bounded position error region, while the orange cylinder represents the LoWC region.}
    \label{fig:validation1}
\end{figure}
The black dashed line represents the nominal contingency landing path of the ego aircraft in the absence of an intruder. A hovering VTOL intruder, shown by the red circle, is then arbitrarily placed such that it intersects the nominal trajectory. The confidence-bounded horizontal and vertical position errors of the intruder are geometrically represented by the red cylindrical volume surrounding the intruder. The encompassing orange cylinder indicates LoWC region. Subsequently, a conflict-aware contingency landing path is replanned, resulting in the trajectory shown by the solid blue line. This path is swept by a shaded blue corridor from start to end, whose size is determined by the confidence-bounded ego position error. In response to the intruder interception, the planner deviates from the nominal trajectory to maintain well-clear separation under position uncertainty. The inset in the upper-right corner provides a magnified view of the closest encounter between the position error boundaries, with the Closest Point of Approach (CPA) indicated by the magenta line.

The second scenario is intended to evaluate conflict avoidance in contingency landing planning when admissible and inadmissible heuristic cost functions generate optimal and suboptimal solutions, respectively. Specifically, the remaining distance to the goal, a commonly used admissible cost, is used as the only heuristic for the optimal case, whereas the suboptimal case uses the default inadmissible cost function, which includes optimal glide and direction of approach terms. Unlike the first scenario, the intruder moves North at constant altitude and speed, as presented in Fig.~\ref{fig:validation2}.
\begin{figure*}[t!]
    \centering
    \includegraphics[width=\linewidth]{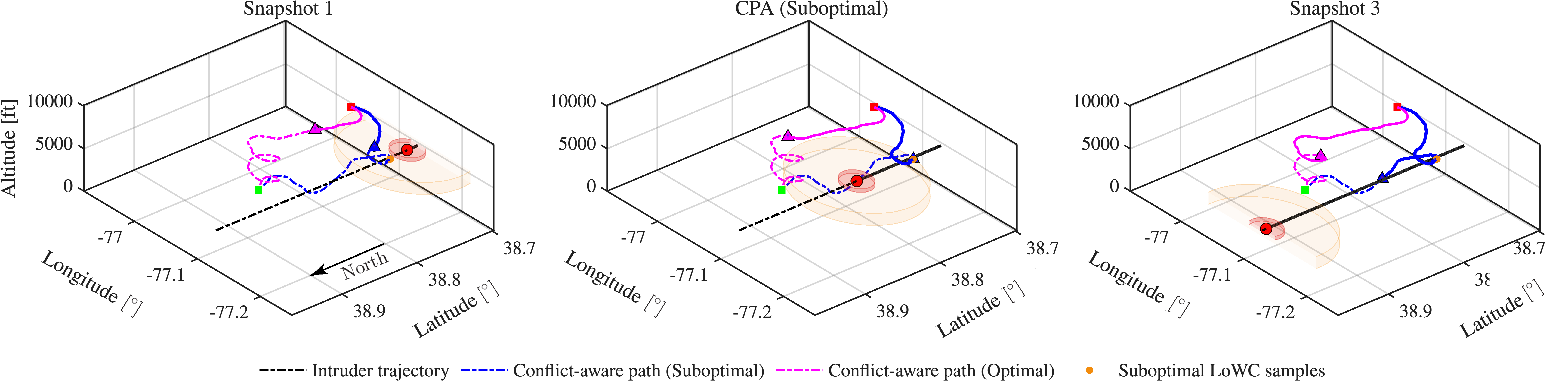}
    \caption{A validation case comparing optimal and suboptimal contingency landing solutions with an intruder in cruise flight. The red cylinder represents the confidence-bounded position error region, while the orange cylinder represents the LoWC region.}
    \label{fig:validation2}
\end{figure*}
From left to right, each plot shows the evolution of the flights forward in time. The middle plot corresponds to the snapshot at which the intruder and the ego executing the suboptimal path, shown by the blue line, reach the CPA. At the CPA, LoWC occurs for the suboptimal trajectory, while the optimal trajectory maintains well-clear separation throughout the flight. This behavior is expected. However, planning runtimes of 4 seconds and 1.1 seconds are registered for the optimal and suboptimal solutions, respectively. As the complexity increases with increasing number of agents, the optimal solution runtime is expected to increase while suboptimal planning remains low-latency. In contingency risk management, the availability of a suboptimal solution may be preferable to solution optimality with higher computational overhead.

\subsection{Conflict-Aware Contingency Landing Planning Benchmark}
The path planning benchmark is performed against geometric Dubins-based solutions. The planner used in this study was originally developed within a generic discrete search framework; however, when flight path angle and constant turn radius constraints are imposed through a finite action set to avoid a high-dimensional state-space, many search- and sampling-based alternatives effectively reduce to generating feasible geometric motion primitives similar in structure to the proposed approach. In addition, an optimal control based method was previously benchmarked against the underlying planner in \cite{heejin2026}, where it was shown to be more suitable for pre-flight planning and computationally prohibitive for real-time implementation. Therefore, the comparison in this work is focused on a representative geometric baseline rather than on another search- or sampling-based planner, or on an optimal control formulation.

To characterize the multi-agent density used in the benchmarking study, historical ADS-B data are processed for a two-week period. The number of simultaneously airborne aircraft is counted at each second and aggregated into 20 minute windows. Per Fig.~\ref{fig:traffic_dist}, the resulting occupancy distribution ranges from 1 to 21 simultaneous airborne aircraft. Most windows contain approximately 6--11 aircraft, with a median of 8 and a 90th percentile of 13. The right tail captures occasional high-density periods with more than 15 simultaneously aloft aircraft, providing a representative range of multi-agent activity for evaluating real-time contingency motion planning and deconfliction.
\begin{figure}
    \centering
    \includegraphics[width=\linewidth]{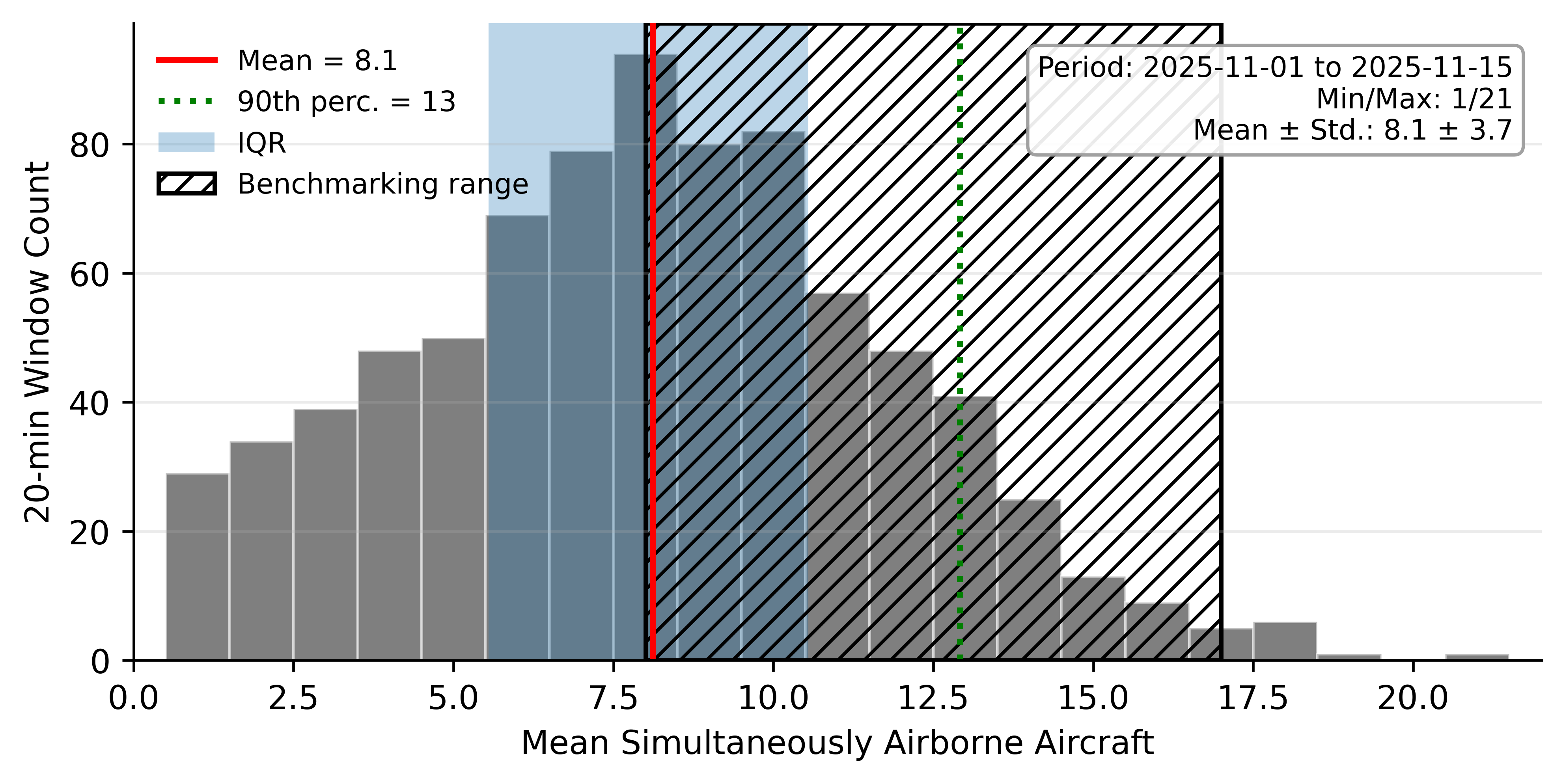}
    \caption{Distribution of simultaneous airborne agents aggregated over 20 minute windows. The shaded region indicates the interquartile range, while the hatched band marks the multi-agent density range selected for benchmarking.}
    \label{fig:traffic_dist}
\end{figure}
Aircraft count range of the 20-minute window chosen for benchmarking is also highlighted in Fig.~\ref{fig:traffic_dist}. For the selected benchmarking window, the first and third quartiles of simultaneous aloft aircraft occupancy are 11 and 14, respectively, with a peak count of 17.

We first present a representative use case. The scenario in Fig.~\ref{fig:use_case} evaluates conflict-aware contingency landing planning in dynamic terminal area traffic. The search-based solution is compared with the baseline Dubins solution under the same traffic realization. In both cases, LoWC is detected from the logged relative states between the ego aircraft and surrounding intruders. The black curve denotes the full ego contingency landing trajectory, while the yellow segments indicate portions of the ego trajectory where LoWC is recorded. The red and purple trajectories correspond to conflicting intruders, with shaded circular regions showing representative swept LoWC regions along the intruder paths. Other nearby traffic is shown using gray markers and one-minute trailing histories.
\begin{figure}[ht!]
    \centering
    \includegraphics[width=\linewidth]{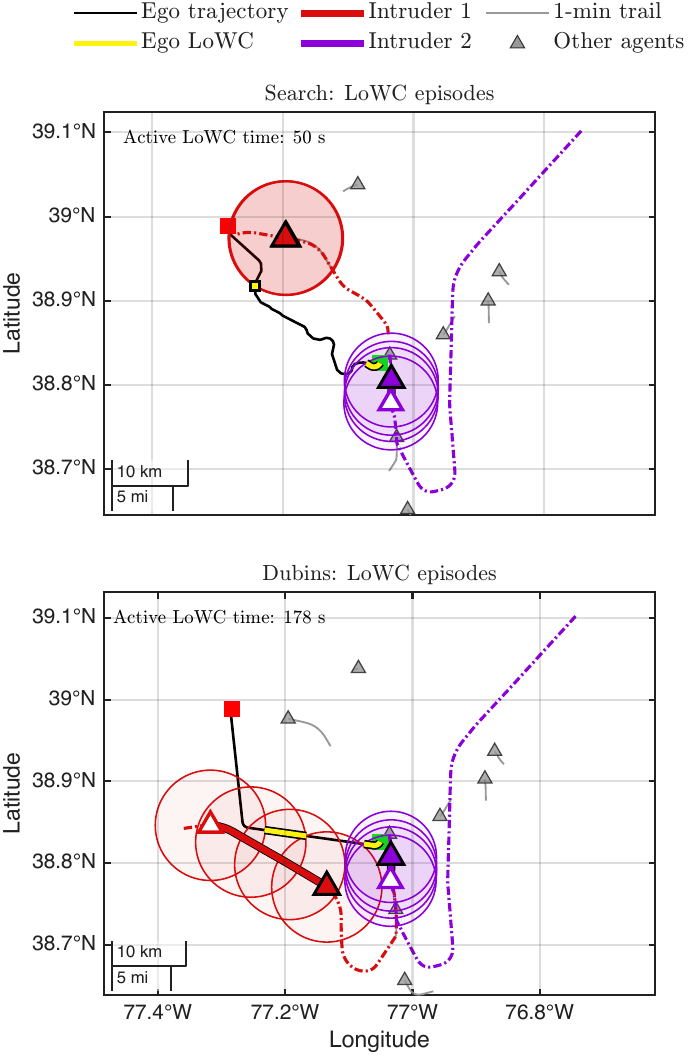}
    \caption{Comparison of the search-based contingency landing solution and the baseline Dubins solution under dynamic terminal-area traffic. The yellow segments indicate logged LoWC intervals. The shaded circles represent swept LoWC regions along intruder trajectories.}
    \label{fig:use_case}
\end{figure}
The search solution produces two LoWC episodes: a short two-second interaction with the first intruder and a later interaction near the terminal portion of the landing trajectory. The latter arises from the bottleneck structure of the final approach region. Near the landing site, all feasible trajectories must converge toward the touchdown point, causing otherwise distinct solutions to collapse into a narrow corridor. If an intruder occupies this corridor at the same time, the conflict becomes unavoidable. Therefore, these later LoWC episodes occur because of an over-constrained contingency landing problem in a multi-agent environment operating below separation minima. The search-based planner reduces active LoWC whenever spatial and temporal freedom exists; however, full separation recovery in the final approach bottleneck requires coordinated resolution advisories beyond path selection alone.

The benchmark study covers a wide range of conflict encounters in a real urban airspace. More than 900 contingency landing trajectories are computed for a wing-lift ego aircraft under complete loss-of-thrust, corresponding to more than 140 hours of gliding flight in a multi-agent environment. Initial contingency positions are randomized over a $28~\mathrm{NM} \times 28~\mathrm{NM}$ area at altitudes ranging from 3,000 ft to 10,000 ft. Figure~\ref{fig:cases} shows the benchmarked contingency and intruder trajectories. The convergence of contingency trajectories indicates the locations of the landing sites.
\begin{figure}[ht!]
    \centering
    \includegraphics[width=\linewidth]{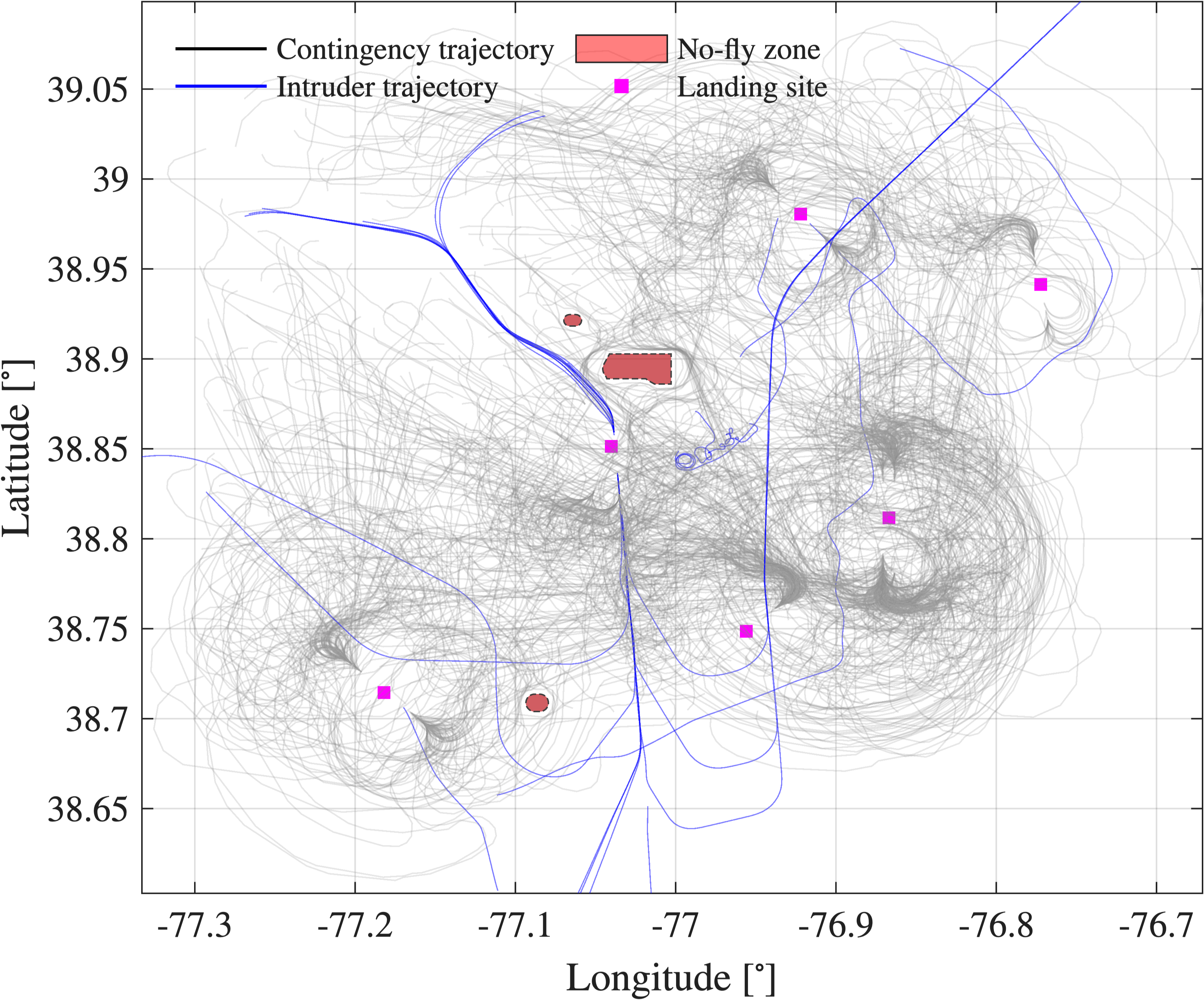}
    \caption{Spatial distribution of forced landing and dynamic agent trajectories.}
    \label{fig:cases}
\end{figure}
Fig. \ref{fig:planning_benchmark} respectively demonstrates cumulative distribution of the LoWC duration, conflict severity, and runtime statistics from top to bottom.
\begin{figure}[ht!]
    \centering
    \includegraphics[width=\linewidth]{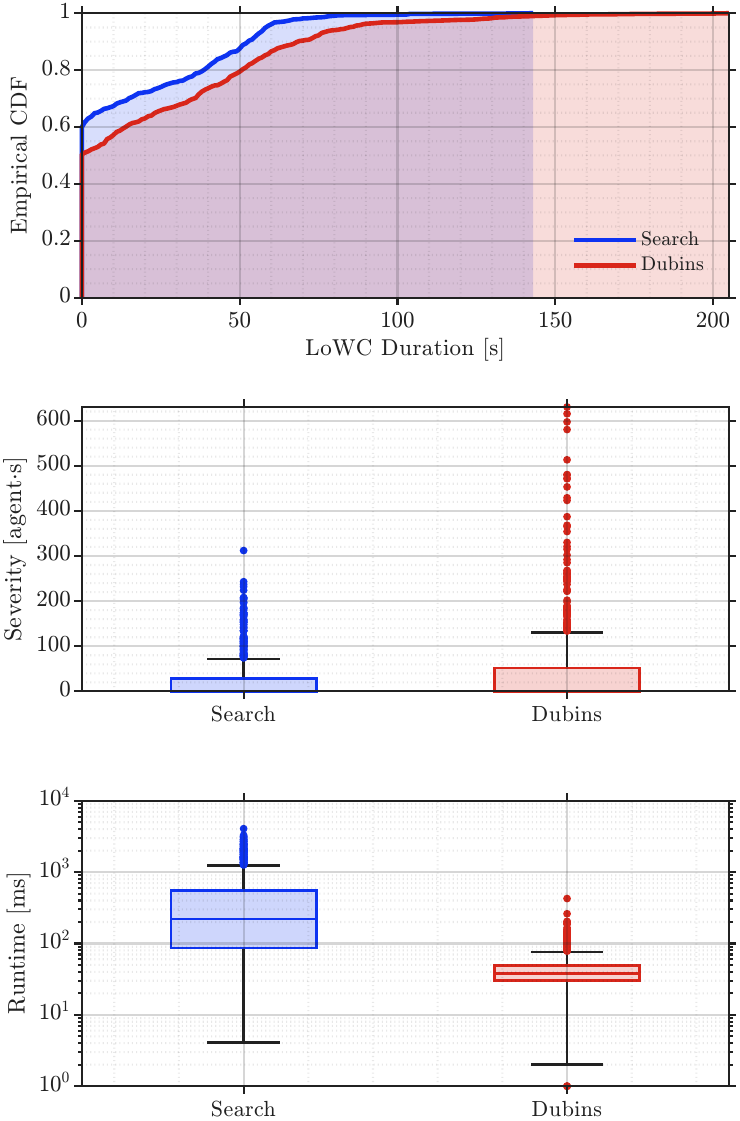}
    \caption{Conflict-aware contingency landing planning benchmark statistics.}
    \label{fig:planning_benchmark}
\end{figure}
Conflict severity is measured as the product of the time spent in LoWC and the number of conflicting intruders, providing an indicator of the intrusiveness of the required resolution advisories. The search-based planner achieves conflict-free outcomes in $60\%$ of the cases, compared with $50.4\%$ for the Dubins baseline. It reduces the mean LoWC duration from 23.3 s to 14.5 s and decreases the total LoWC exposure from 6 h to 3.7 h, corresponding to a $37.7\%$ reduction. The maximum number of conflicting intruders also decreases from 6 to 4. Overall, the search-based planner reduces the accumulated conflict severity by $45.5\%$, from 10.6 to 5.8 $\mathrm{agent}\cdot\mathrm{h}$. However, this improvement comes at the cost of increased computational burden: the search-based planner averages 455 ms with a worst-case runtime of 4 s, whereas the Dubins solver has mean and worst-case runtimes of 40 ms and 400 ms, respectively. The higher-end runtimes correspond to cases involving left and right S-turn extensions, which require evaluating eight candidate paths and integrating the cost along each solution.

The results show that the search-based planner statistically improves conflict avoidance during contingency landing in a multi-agent environment compared with the Dubins baseline, for which the candidate with the lowest cumulative LoWC risk in Eq.~\eqref{eq:cumul_lowc_risk} is selected. However, the search-based planner is not uniformly superior. In $19\%$ of all cases, the Dubins solution yields lower conflict severity. These cases reflect favorable encounter geometries produced by the Dubins candidate set, whereas the search-based planner prioritizes low overhead replanning over global optimality. Therefore, the trajectory with the lowest predicted conflict severity among the available candidates is selected for execution.

\subsection{Conflict Resolution Advisory Generation}
This section evaluates the proposed hierarchical conflict resolution advisory framework given in Algorithm \ref{alg:parallel_resolution}. For each individual case, the earliest advisory onset is lower bounded by the accumulated time required for path computation, datalink communication, and advisory generation, thereby simulating a serial PAA execution including communication for coordination. The distressed ego shares its landing intent after onboard planning by transmitting the selected trajectory through a datalink to the processor assigned for advisory generation. The resulting deconfliction is then communicated back through the datalink to the conflicting intruders. The total datalink delay is conservatively set to 1 s, corresponding to 500 ms in each direction since high-level command transmission latency for UAS application is reported less than 500 ms in \cite{9075984}, and less than 50 ms with 4G/5G technology in \cite{8935306}. Optimization runtime is limited to 1 second for conflict resolution.

First, the contingency landing use case presented in Fig. \ref{fig} is revisited to resolve the conflicts associated with the search-based solution. The resulting advisories are shown in Fig. \ref{fig}.
\begin{figure}[t!]
\centering
\includegraphics[width=\linewidth]{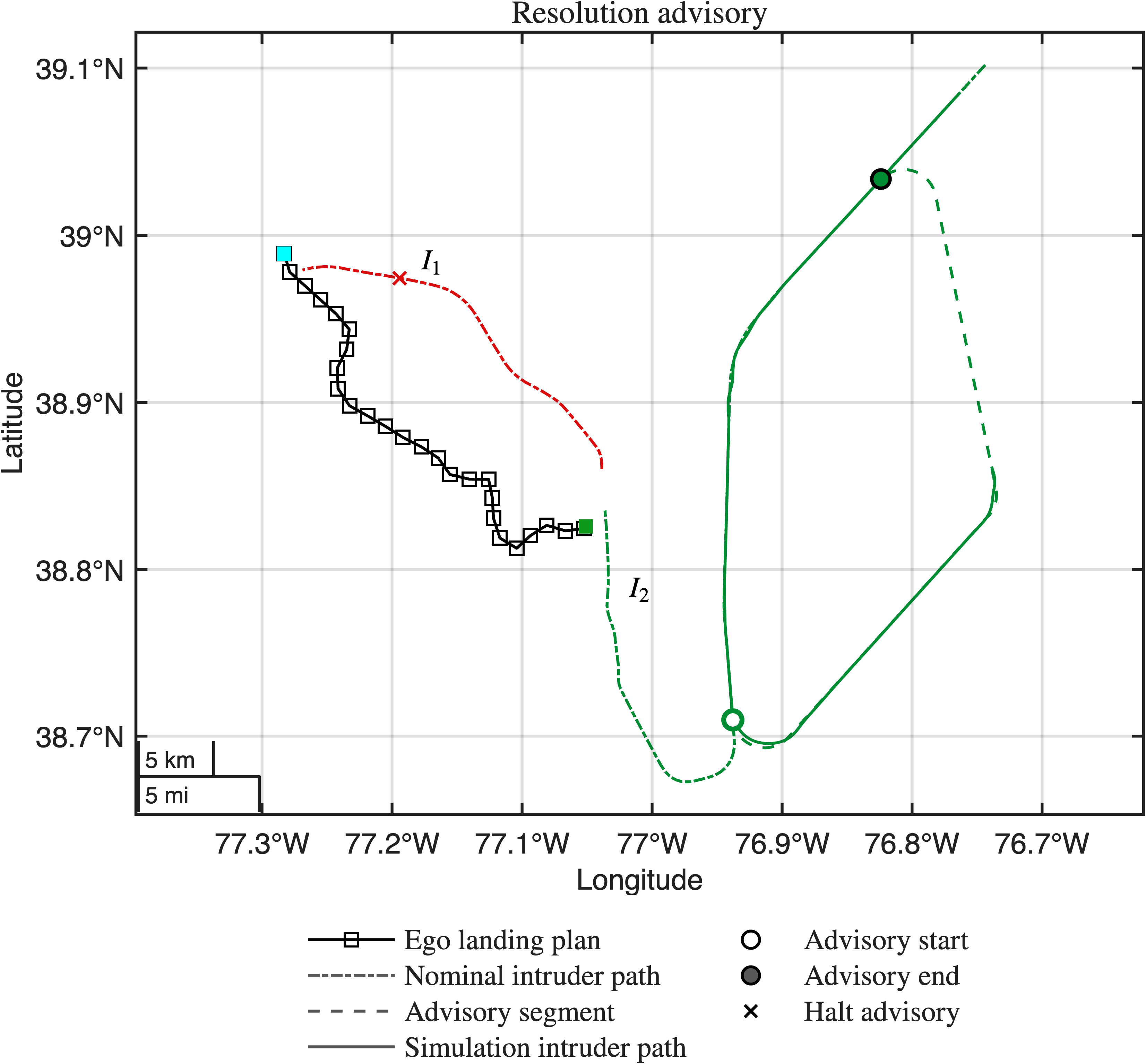}
\caption{Resolution advisories for the conflicts occurring in the contingency landing path planning use case.}
\label{fig}
\end{figure}
The cyan and green markers respectively indicate the initial and final states of the contingency plan. Each intruder is color-coded: the first intruder, denoted by $I_1$, is shown in red, whereas the second intruder, denoted by $I_2$, is shown in green. The dash-dotted lines represent the nominal intruder trajectories before deconfliction, the dashed green line represents the advised trajectory for $I_2$, and the solid green line denotes the corresponding dynamic simulation trajectory of $I_2$. Since $I_1$ is scheduled to take off only seconds after the contingency onset, it is advised to delay its departure. This delay postpones its flight beyond the landing time horizon; therefore, no associated advised trajectory is shown for $I_1$. In contrast, $I_2$ is approaching for landing at the time of the ego loss-of-thrust event. Given the predefined advisory hierarchy, the first feasible solution is found by extending the nominal trajectory through $\mathcal{R}_{\mathrm{Extend}}$. In other words, no feasible solution is found within the 1 s runtime limit for either $\mathcal{R}_{\mathrm{Speed}}$ or $\mathcal{R}_{\mathrm{Alt}}$. The resulting S-turn trajectory, shown by the green advised path, departs from the nominal path at the advisory start and later reconnects to it.

The separation margins associated with this scenario are shown in Fig. \ref{fig:460_margins}. Conflicts with $I_1$ and $I_2$ occur around 200 s and 700 s, respectively, where $\delta_H -  R_H \leq 0$ and $\delta_V -  R_V \leq 0$ simultaneously. Since no advised trajectory is computed for $I_1$, its post-advisory separation curves are not shown. For $I_2$, the green curves compare the nominal, advised, and simulated well-clear margins, with the advised and simulated trajectories deviating from the nominal case as a result of the path-extension advisory. Simulation results agree well with the kinematics-based advisory, with only minor offsets caused by transient response effects.
\begin{figure}[t!]
\centering
\includegraphics[width=\linewidth]{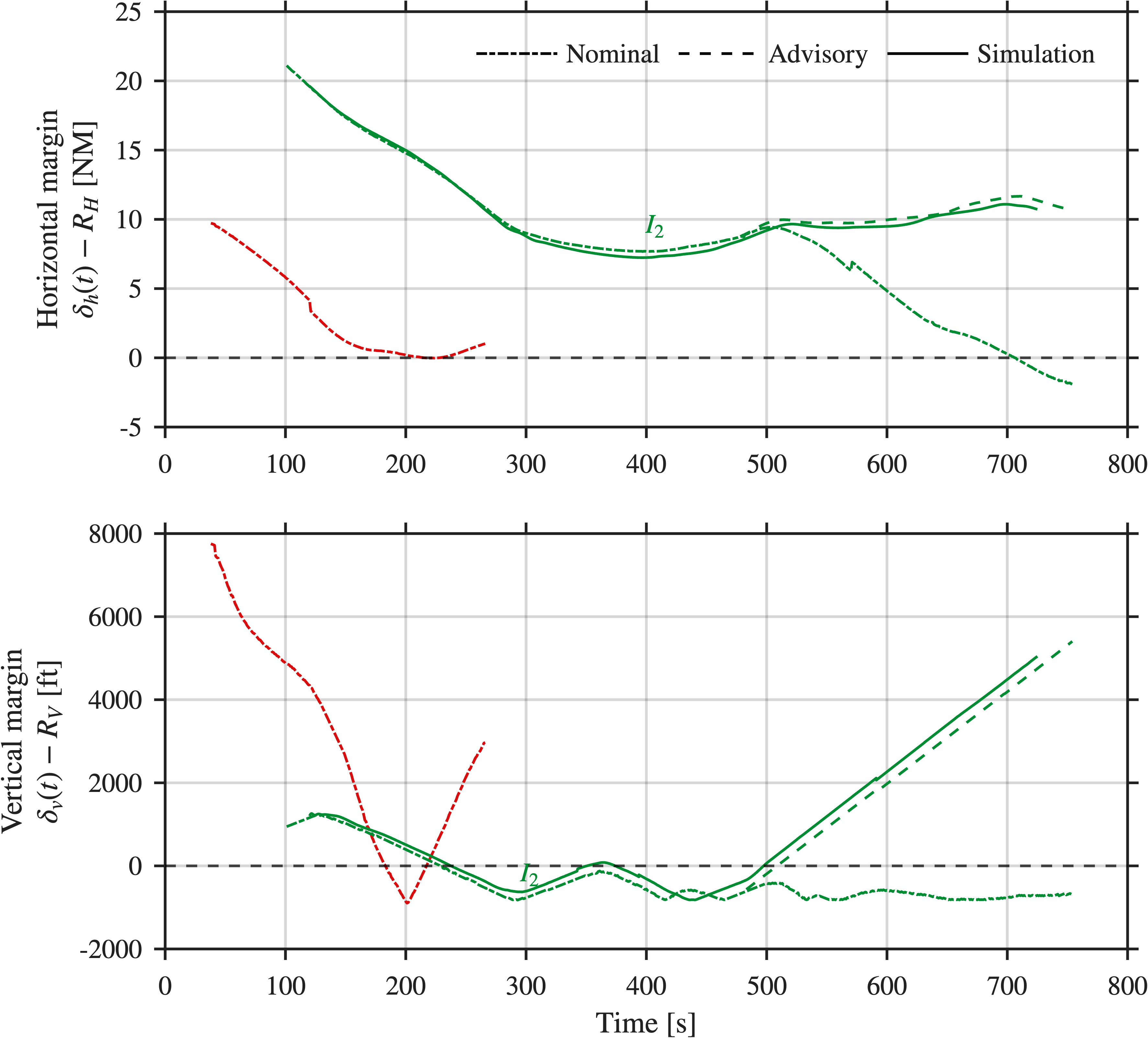}
\caption{Separation margins before and after deconfliction with dynamic simulation comparison.}
\label{fig:460_margins}
\end{figure}
For this specific case, the search-based contingency landing trajectory computation runtime is logged as 630 ms. As $\mathcal{R}_{\mathrm{Halt}}$ runtime is trivial, only $\mathcal{R}_{\mathrm{Extend}}$ contributes to the total computational expense, resulting in a total PAA time of 2.6 s after including the two-way communication delay.

Another example illustrating deconfliction through $\mathcal{R}_{\mathrm{Alt}}$ and $\mathcal{R}_{\mathrm{Divert}}$ is shown in Fig. \ref{fig:30_resolution}, including the corresponding dynamic simulation tracks.
\begin{figure}[t!]
\centering
\includegraphics[width=\linewidth]{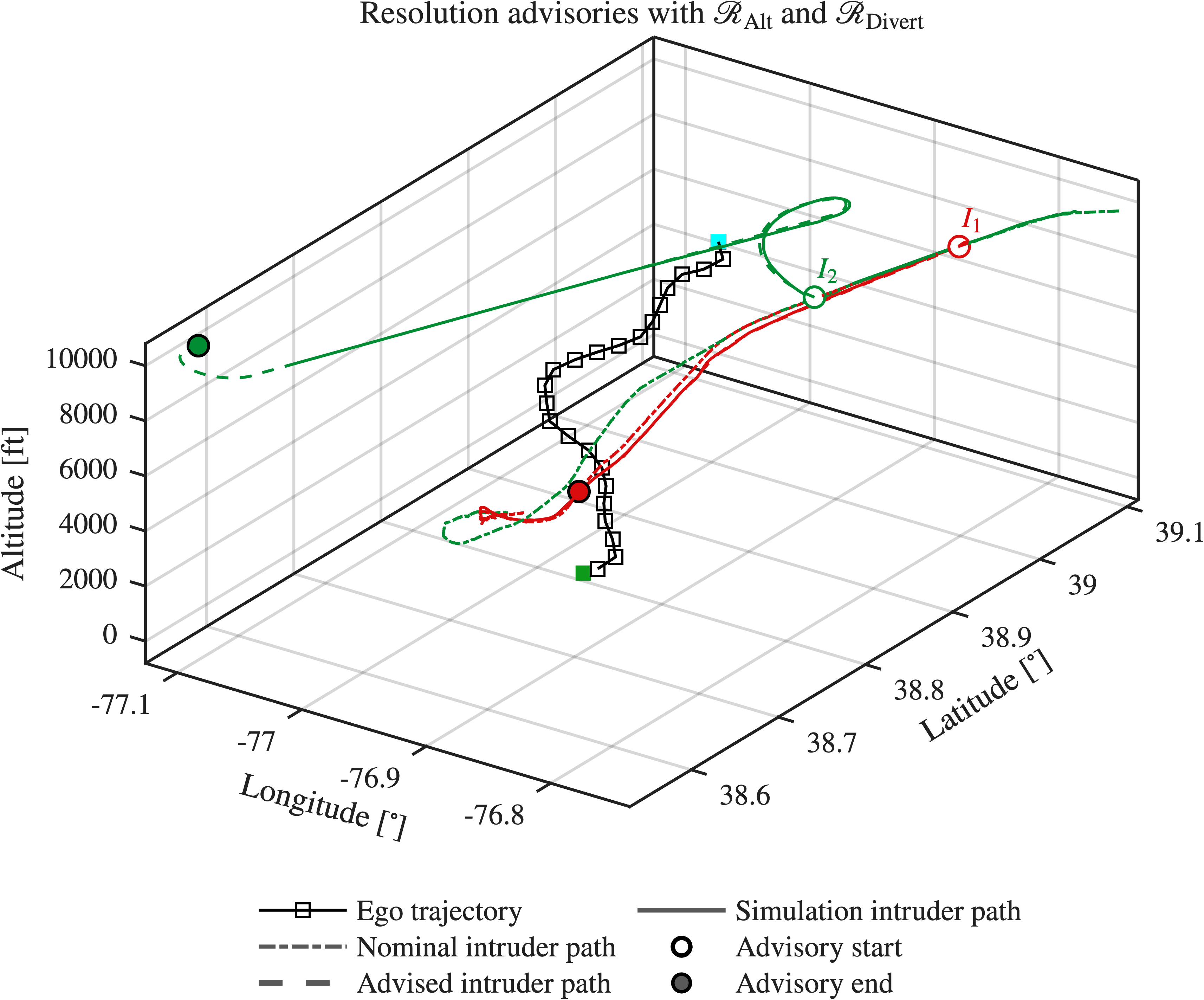}
\caption{Altitude and divert resolution advisories with dynamic realizations.}
\label{fig:30_resolution}
\end{figure}
In this scenario, well-clear separation with $I_1$ is restored in the vertical plane through $\mathcal{R}_{\mathrm{Alt}}$. The advisory commands a descent of $\Delta h_{\mathrm{adv}} = 390$ ft at $\gamma_1 = -3.5^\circ$, followed by a hold at the offset altitude for $t_{\mathrm{hold}} = 70$ s and a subsequent climb at $\gamma_2 = 0.97^\circ$. For intruder $I_2$, $\mathcal{R}_{\mathrm{Divert}}$ is issued because no less intrusive solution is found within the specified runtime limit. The corresponding separation curves are shown in Fig. \ref{fig:30_separation}.

Because $I_1$ is modified only in the vertical plane, the horizontal margins before and after the advisory largely overlap, and the dynamic simulation closely agrees with the advised response. The effect of the advisory is instead reflected in the increased vertical margin. For $I_2$, the divert advisory affects both the horizontal and vertical margins.
\begin{figure}[t!]
\centering
\includegraphics[width=\linewidth]{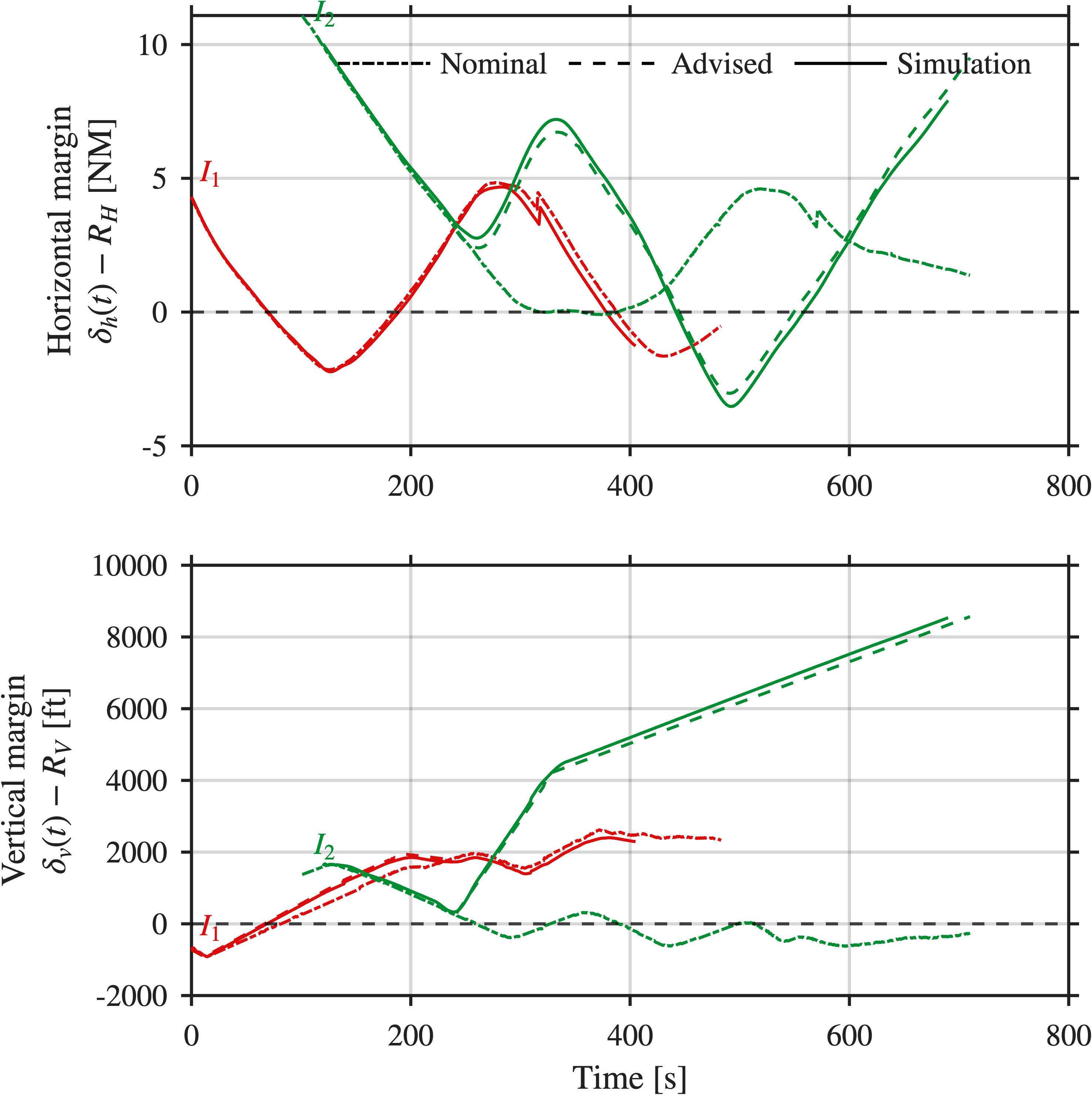}
\caption{Separation margins for the case shown in Fig. \ref{fig:30_resolution}.}
\label{fig:30_separation}
\end{figure}
Overall, the total coordinated response time for this case is approximately 4 s, with path planning, communication, and advisory generation each contributing on the order of one second.

Next, all the LoWC encounters in the planning benchmark solutions are resolved for statistical evaluation. Table~\ref{tab:resolution_summary} summarizes the results. The proposed method returned feasible resolutions for all 575 conflicts, indicating that the predefined advisory set is sufficient for the evaluated scenarios. Most conflicts involved airliner, followed by VTOL intruders, while smaller wing-lift accounted for a minor portion of the encounters. Among the generated advisories, speed regulation is the most common resolution, suggesting that many conflicts could be resolved through temporal separation without requiring spatial trajectory modification. Halt advisories are also frequent because of the chosen ADS-B dataset, whereas altitude, extension, and diversion advisories were used less often when simpler interventions are insufficient.
\begin{table}[t!]
\centering
\caption{Summary of conflict-resolution advisory generation results.}
\label{tab:resolution_summary}
\begin{tabular}{lcc}
\hline \hline
\textbf{Category} & \textbf{Count} & \textbf{Percentage} \\
\hline
\multicolumn{3}{l}{\textit{Resolution outcome}} \\
Feasible & 575 & 100.0$\%$ \\
Infeasible & 0 & 0.0$\%$ \\
\hline
\multicolumn{3}{l}{\textit{Intruder type}} \\
Airliner & 343 & 59.7$\%$ \\
AAM wing-lift & 39 & 6.8$\%$ \\
AAM VTOL & 193 & 33.6$\%$ \\
\hline
\multicolumn{3}{l}{\textit{Advisory type}} \\
Halt & 157 & 27.3$\%$ \\
Speed & 208 & 36.2$\%$ \\
Altitude & 71 & 12.3$\%$ \\
Extend & 36 & 6.3$\%$ \\
Divert & 103 & 17.9$\%$ \\
\hline
\end{tabular}
\end{table}

\begin{table}[t!]
\centering
\caption{Optimized advisory parameter statistics.}
\label{tab:advisory_parameters}
\renewcommand{\arraystretch}{1.2}
\resizebox{\linewidth}{!}{
\begin{tabular}{llcccc}
\hline \hline
\textbf{Advisory type} & \textbf{Parameter} & \textbf{Mean} & \textbf{Std. dev.} & \textbf{Min.} & \textbf{Max.} \\
\hline
\multirow{2}{*}{VTOL speed}
& $\Delta v_{\mathrm{adv}}$ [kts] & -14.5 & 56.5 & -107.1 & 99.0 \\
& $t_{\mathrm{hold}}$ [s] & 131.6 & 96.8 & 5.5 & 485.3 \\
\hline
\multirow{2}{*}{Wing-lift speed}
& $\Delta v_{\mathrm{adv}}$ [kts] & -14.0 & 17.1 & -26.1 & 25.7 \\
& $t_{\mathrm{hold}}$ [s] & 278.1 & 96.0 & 155.4 & 519.8 \\
\hline
\multirow{4}{*}{Altitude} & $\Delta h_{\mathrm{adv}}$ [ft] & 348.1 & 389.5 & -783.6 & 998.9 \\
& $t_{\mathrm{hold}}$ [s] & 128.8 & 93.1 & 1.5 & 393.6 \\
& $\gamma_0$ [deg] & -0.4 & 2.2 & -4.4 & 2.9 \\
& $\gamma_1$ [deg] & -2.7 & 1.6 & -4.7 & 2.6 \\
\hline
Extend & $L$ [NM] & 12.8 & 7.8 & 0.9 & 32.4 \\
\hline
\multirow{2}{*}{Divert} & $L$ [NM] & 23.4 & 11.6 & 9.6 & 40.8 \\
& $\gamma$ [deg] & 2.9 & 0.8 & 0.2 & 4.9 \\
\hline
\end{tabular}
}
\end{table}
Table~\ref{tab:advisory_parameters} summarizes the optimized advisory decision variables for resolutions. The mean speed and altitude changes remain relatively small, indicating that many conflicts are resolved with modest temporal or vertical adjustments. In contrast, extension and diversion advisories are more intrusive since they modify the path length and can therefore affect a larger portion of the intruder mission. The wide parameter ranges also show that the required advisory magnitude is encounter-dependent, varying with intruder type, conflict timing, and contingency landing trajectory.

Moreover, advisory lead time is a crucial parameter for DAA well-clear because advisories must be generated sufficiently early to allow the surrounding aircraft to interpret, accept, and execute the commanded maneuver. RTCA DO--365 Minimum Operational Performance Standards for Detect-and-Avoid Systems uses a 35 s modified tau threshold as part of the DAA well-clear alerting criteria \cite{RTCA_DO365_2017,Wu2020}. Therefore, Table~\ref{tab:onset_to_conflict_stats} reports the advisory lead time, defined as the time from advisory onset to the predicted conflict, and the percentage of advisories generated at least 35 s before the predicted conflict. Overall, 93.5$\%$ of all advisories satisfy this 35 s threshold. These results indicate that the proposed advisory framework generally produces conflict resolution actions within a DAA-relevant temporal horizon.
\begin{table}[t!]
\centering
\caption{Descriptive statistics of advisory lead time $t_{c} - t_{\mathrm{on}}$. All units are in seconds.}
\label{tab:onset_to_conflict_stats}
\begin{tabular}{lccccc}
\hline\hline
\textbf{Advisory}& \textbf{Mean} & \textbf{Std. dev.} & \textbf{Min.} & \textbf{Max.} & \textbf{$\geq 35$ s [\%]} \\
\hline
Speed & 380.1 & 91.7 & 137.8 & 592.9 & 100.0 \\
Altitude & 119.4 & 86.9 & 7.7 & 329.7 & 84.5 \\
Extend & 152.1 & 82.6 & 18.0 & 384.1 & 94.4 \\
Divert & 167.5 & 89.7 & 10.0 & 456.0 & 86.4 \\
\textit{All} &\textit{ 263.8} & \textit{147.1} & \textit{7.7} & \textit{592.9} & \textit{93.5} \\
\hline
\end{tabular}
\end{table}
For cases with advisory lead times below 35 s—6.5$\%$ of all cases, the reduced temporal margin is a consequence of the priority trajectory being declared on short notice. In such cases, earlier coordination is not available by definition: the framework does not pre-plan around a priority trajectory before it is declared, but instead generates the least disruptive feasible advisory for the surrounding traffic given the remaining time. These cases would also violate the DO--365 DAA well-clear temporal threshold, indicating that the conflict is already inside the nominal alerting horizon when the priority trajectory becomes available.

\subsection{Solution Optimality and Computational Performance}
As aforementioned, optimization execution runtime is limited to 1 s for conflict resolution generation for the presented application. This is an early stopping approach to prevent the high computation expense of the optimal advisories, prohibitive for a real-time implementation. It is a similar problem encountered in the search-based trajectory generation for which inadmissible heuristics are employed for suboptimal yet substantially faster convergence. Therefore, optimality is traded for real-time performance that is strictly necessary for forced landing scenarios in multi-agent environments that require swift coordination for safety.

We further investigated how much solution quality is compromised for real-time performance. Fig. \ref{fig:opt_comparison} compares the solutions for the same conflict dataset under 15 s and 1 s optimizer runtime limits.
\begin{figure*}
    \centering
    \includegraphics[width=\linewidth]{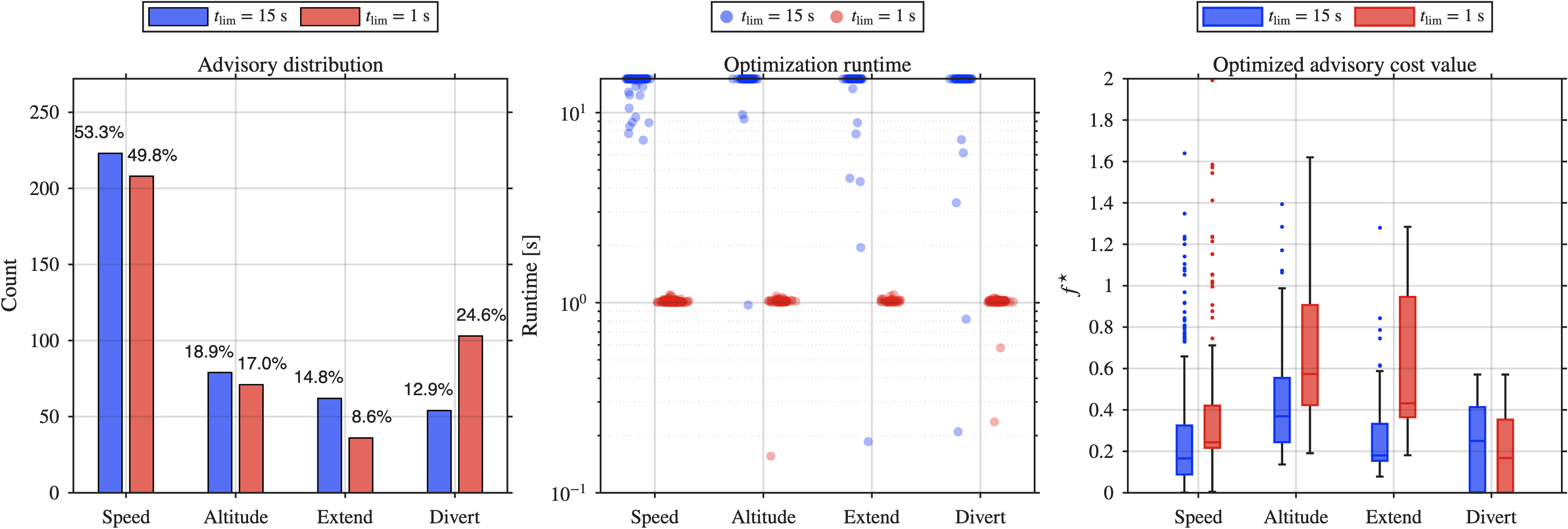}
    \caption{Comparison of optimization outcomes under 15 s and 1 s runtime limits. Advisory distributions are shown excluding halt advisories, with percentages normalized by the total number of plotted advisories for each runtime-limit case.}
    \label{fig:opt_comparison}
\end{figure*}
The comparison suggests that enforcing a 1 s optimizer limit preserves the overall structure of the advisory outcomes while substantially reducing computation time. As expected, the 1 s case exhibits runtime saturation, with most optimization calls terminating near the imposed upper bound. In contrast, the 15 s case often requires several seconds and frequently approaches the longer runtime limit. The additional computation, however, does not fundamentally change the dominant resolution modes. Speed regulation remains the most common advisory type, and altitude regulation appears with comparable relative frequency in both cases. The most noticeable change is the increase in diversion advisories under the 1 s limit, accompanied by a reduction in extension advisories. This indicates that early stopping can shift some encounters toward more intrusive feasible advisories in the ordered set. The optimized objective values show that the longer runtime can improve individual solutions, particularly by reducing the cost for some speed, altitude, and extension advisories. Yet, the aggregate improvement is modest relative to the increase in computational burden. The 1 s limit therefore provides a practical real-time compromise: it maintains a similar advisory distribution and acceptable objective quality while enforcing the latency needed for conflict resolution in forced landing scenarios with multi-agent coordination.

Figure~\ref{fig:paa_runtime} shows the end-to-end runtime of the Plan--and--Avoid timeline, including path planning, datalink delays, and resolution advisory (RA) generation. Path planning runtime is reported as the mean plus three standard deviations, while RA runtime is reported as the worst case across test cases. After the priority trajectory is planned, advisory candidates are evaluated in parallel.
\begin{figure}[t!]
    \centering
    \includegraphics[width=\linewidth]{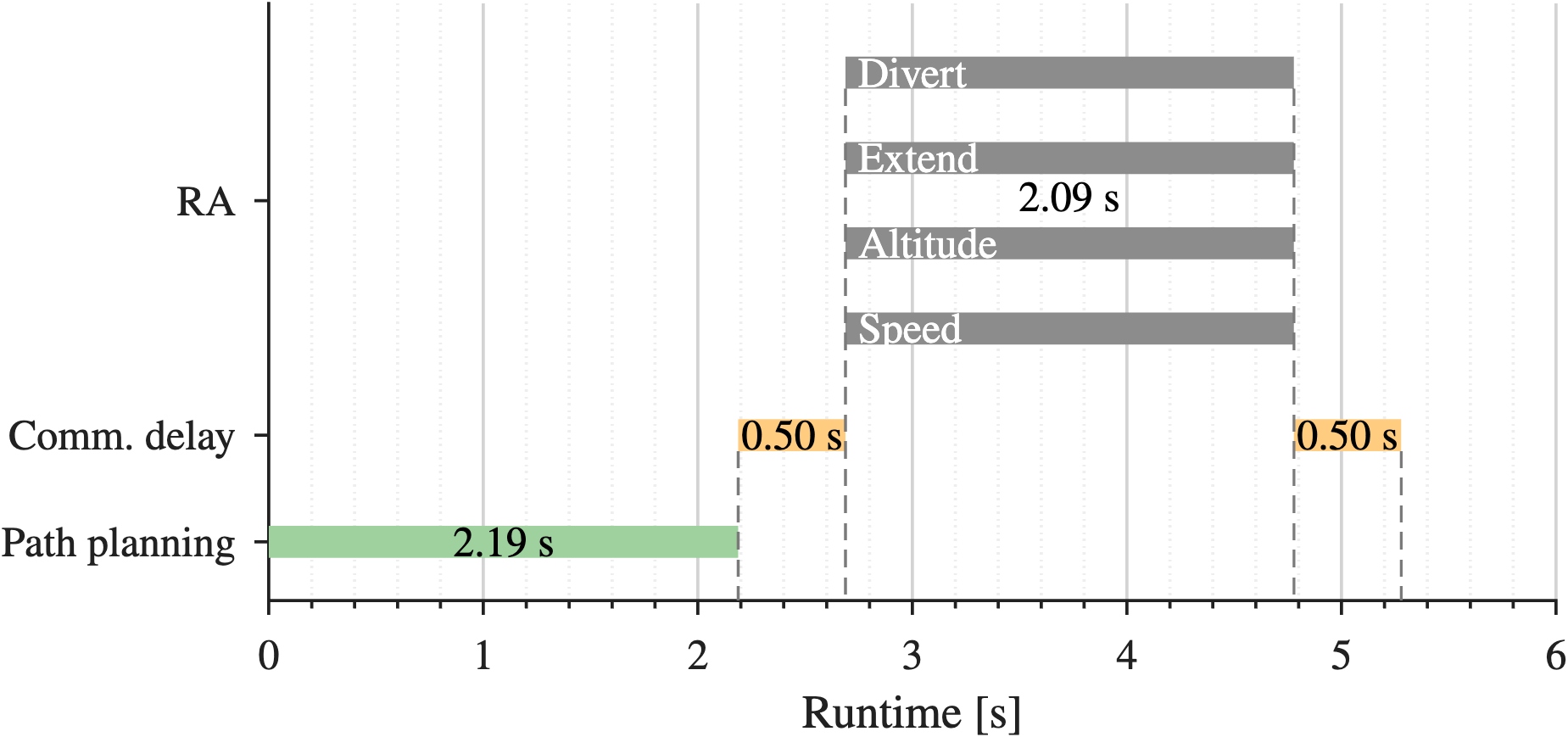}
    \caption{Plan-and-Avoid execution runtime: path planning runtime is reported as mean plus three standard deviations; resolution advisory (RA) runtime is reported as the worst case across test cases.}
    \label{fig:paa_runtime}
\end{figure}
Finally, end-to-end PAA runtime distribution is plotted in Fig. \ref{fig:paa_runtime_dist}.
\begin{figure}[t!]
    \centering
    \includegraphics[width=\linewidth]{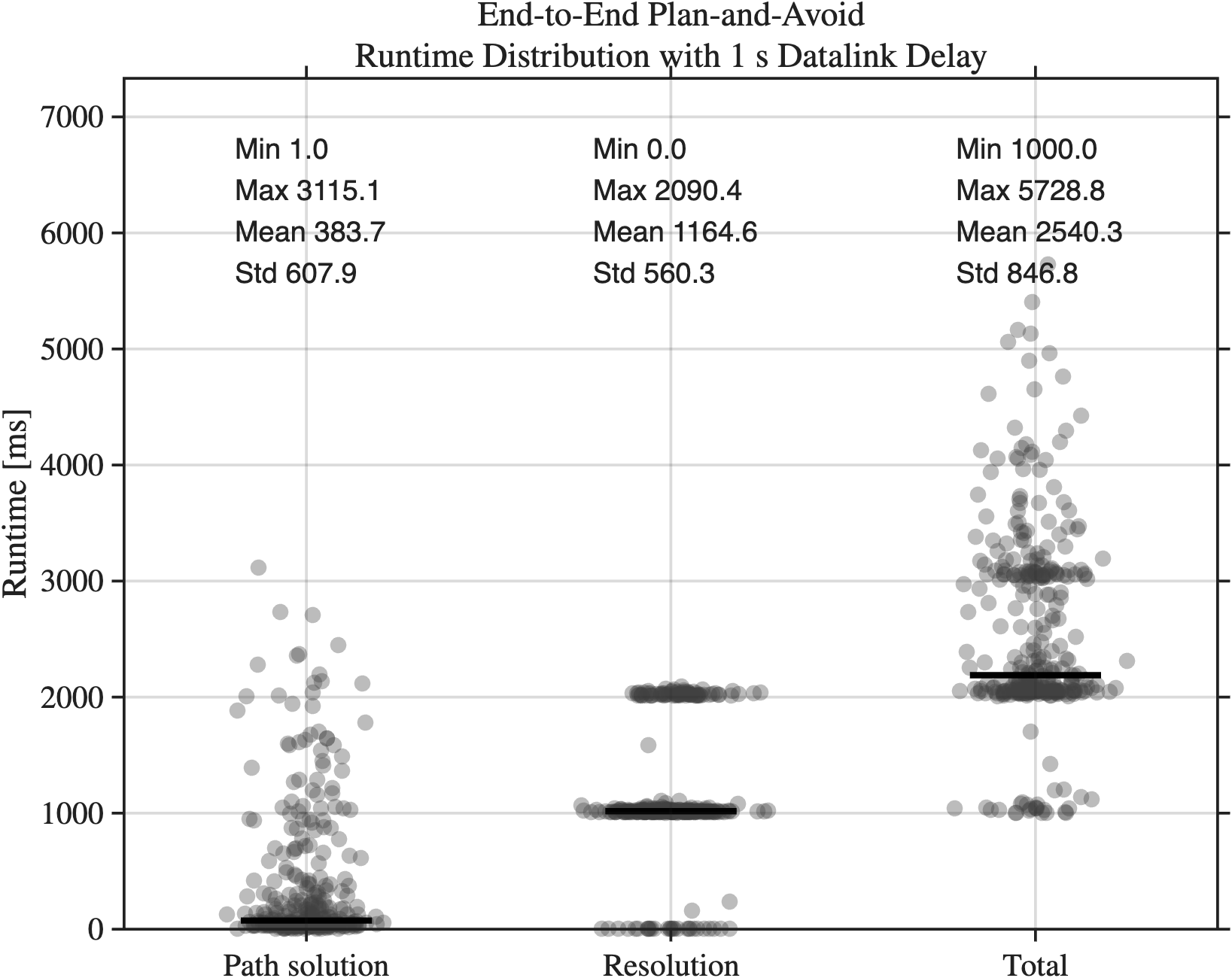}
    \caption{End-to-end runtime distribution of the Plan-and-Avoid pipeline for cases with LoWC occurrence, including a two-way 1 s datalink delay.}
    \label{fig:paa_runtime_dist}
\end{figure}
The runtime results demonstrate that the proposed PAA framework can generate conflict-aware contingency responses within a few seconds, even after including a 1 s datalink delay. This fast response is essential in forced landing scenarios, where the aircraft state, reachable landing options, and surrounding traffic evolve continuously. From an operational perspective, this is the key result: the system does not merely find conflict-aware solutions, but does rapidly to be relevant during an emergency.

%% file: Sections/7_Conclusion.tex
\section{Discussion}
\label{sec:disc}
The proposed PAA method is intended for cooperative or managed airspace where surrounding agents can share trajectory intent and receive high-level advisories through datalink communication. This assumption is consistent with emerging concepts for coordinated AAM, UAS traffic management, and fleet-operated autonomous aircraft, and is critical for safety in densely occupied airspace volumes.  However, this assumption does not support noncooperative scenarios. Also, the ADS-B trajectories used in the benchmark should be interpreted as realistic dynamic agent motion traces rather than a complete representation of future sUAS or AAM traffic. Their value lies in providing noisy, time-varying, multi-agent airspace interactions at a scale that is difficult to reproduce through flight experiments.

The case given in Fig. \ref{fig:30_resolution} illustrates the role of the proposed PAA framework as a decision-support tool for traffic management rather than as a last-second collision avoidance system. Recall that TCAS- and ACAS-X-type systems are effective for reactive separation recovery during nominal operations, but they are primarily designed around pairwise, reciprocal advisories between cooperative aircraft. The proposed PAA framework accommodates new or modified priority trajectories by generating unilateral advisories for surrounding traffic in a form consistent with current Air Traffic Control (ATC) practice: one intruder, $I_1$, is assigned a modified lower-altitude profile to restore vertical separation, while the second intruder, $I_2$, is directed to enter a holding maneuver to remove the temporal conflict. These actions are interpretable, constraint-aware, and operationally recognizable, resembling the type of instructions a controller could issue after assessing the encounter geometry. The key advantage is that the assessment and advisory selection are performed in milliseconds considering disruption, allowing the system to translate multi-agent conflict predictions into clear ATC-compatible resolutions much faster than human communication and coordination alone. Thus, the framework can be viewed as an explainable decision layer that supports controllers during high-workload contingency events while preserving the structure of existing air traffic procedures.

The proposed Plan--and--Avoid framework should be interpreted as complementary to, rather than a replacement for, conventional Detect--and--Avoid systems. PAA assumes that a priority trajectory is declared and that surrounding agents provide trajectory intent, allowing conflicts along the declared trajectory to be predicted and resolved explicitly through advisories computed and issued typically further in advance than is possible with DAA. However, stochastic effects such as unexpected changes in intruder intent cannot be fully eliminated at the planning layer. In such cases, DAA remains necessary as a downstream safety layer for monitoring the realized encounter geometry and detecting conflicts that emerge after the advisory has been issued. This distinction also clarifies the role of modified tau. In DO--365-based DAA logic, modified tau is used as a temporal alerting metric for identifying loss of DAA well clear. In the proposed framework, modified tau is not directly used as an optimization variable because all predicted conflicts along the priority trajectory are evaluated and resolved within the available planning horizon. Thus, PAA provides proactive coordination based on declared intent, while DAA provides reactive protection against uncertainty and unmodeled deviations.

The results illustrate the tradeoff between optimality and real-time feasibility. Both the contingency planner and the advisory optimizer intentionally sacrifice global optimality to satisfy the latency requirements of forced landing scenarios. This is not merely a computational convenience; it is part of the safety argument. A globally optimal advisory that requires excessive computation may be operationally irrelevant during an emergency, whereas a feasible and interpretable advisory generated within a few seconds can support timely coordination. The comparison between 1 s and 15 s optimizer limits suggests that longer optimization can improve individual objective values, but the aggregate improvement is modest relative to the additional computational burden. Moreover, the reported runtimes are obtained on a personal computer, representing a conservative onboard or edge-computing implementation. In a deployed traffic management architecture, advisory candidates could be distributed across high-performance ground or cloud-computing resources, allowing larger optimization budgets, denser candidate sampling, or parallel evaluation of multiple advisory classes within the same wall-clock time. Under such computing assumptions, near-optimal or even globally optimal advisory solutions may be achievable within a 1 s decision window for many encounters. Therefore, the present implementation should be interpreted as a latency-constrained baseline: it demonstrates that feasible and operationally interpretable advisories can be generated in real time without relying on specialized computing, while leaving additional optimality gains available through parallel or cloud-enabled deployment.

\section{Conclusions and Future Work}
\label{sec:conc}
This paper establishes Plan--and--Avoid as a real-time coordination framework for managing priority trajectories in cooperative multi-agent airspace. The main contribution is the integration of two complementary capabilities: conflict-aware priority trajectory generation and vehicle-constrained advisory generation for surrounding traffic. The framework is evaluated using forced-landing priority trajectories with real-world airspace data and dynamic multi-agent simulations.  Results show that conflict-aware search reduces loss-of-well-clear exposure relative to geometric baselines, while the advisory module resolves remaining conflicts through feasible modifications to surrounding traffic. The majority of advisories found for the benchmark set satisfy the RTCA DO-365 Detect--and--Avoid modified tau requirement, supporting the operational relevance of the generated resolutions. Worst-case planning, resolution, and coordination runtime is 5.7 s on a personal computer, demonstrating real-time applicability. Overall, Plan--and--Avoid provides a practical structure for priority trajectory operations by combining fast trajectory evaluation, explainable advisory generation, and system-level coordination through datalink within seconds.

Future work will extend the proposed Plan--and--Avoid framework in several directions. One direction is incorporating intent prediction and anytime replanning, since surrounding agents may deviate from their expected trajectories during priority operations. The framework should also be evaluated from the perspective of surrounding agents, since advisories may introduce delays, additional energy consumption, and/or disruptions that propagate through the broader multi-agent terminal area airspace region.  Additionally, consideration of multiple faults or conflict scenarios will require a tradeoff between contingency plans. That is, if there are imminent collisions predicted with both the aircraft and ground, prioritization needs to be handled carefully. Finally, flight experiments and human-operator studies are needed to assess how the framework can be certified and transitioned to practice.